\documentclass[letterpaper, 10 pt, conference]{ieeeconf}  

\IEEEoverridecommandlockouts                              

\title{\LARGE \bf
Exact and Bounded Collision Probability for Motion Planning under Gaussian Uncertainty
}

\author{Antony Thomas, Fulvio Mastrogiovanni
and Marco Baglietto 
\thanks{The authors are with the Department of Informatics, Bioengineering, Robotics, and Systems Engineering, University of Genoa, Via All'Opera Pia 13, 16145 Genoa, Italy. Corresponding author's email: \tt antony.thomas@dibris.unige.it}%
}

\usepackage{url}
\usepackage{makeidx}         
\usepackage{graphicx}        
\usepackage{multicol}        
\usepackage[bottom]{footmisc}

\usepackage{amsmath,bm,amsfonts,amssymb}

\usepackage{commath}      
\usepackage{color,xcolor,ucs}
\usepackage{subfig}
\usepackage[font=small,labelfont=bf]{caption}
\usepackage{floatrow}
\usepackage{tabularx}
\usepackage{float}
\usepackage{amsthm}

\usepackage{cite}
\usepackage{xr-hyper}
\usepackage{upgreek}
\usepackage[euler]{textgreek}
\usepackage{wrapfig,booktabs}
\usepackage{algpseudocode}
\usepackage[ruled,vlined,linesnumbered]{algorithm2e}
\usepackage{multirow}
\usepackage{listings}
\usepackage{lipsum}
\usepackage{textcomp} 

\usepackage[nocomma]{optidef}  
\SetKwInOut{Input}{input}
\SetKwInOut{Output}{output}

\newtheorem{theorem}{Theorem}
\newtheorem{definition}{Definition}
\newtheorem{lemma}{Lemma}

\DeclareMathOperator{\EX}{\mathbb{E}} 

\newcommand{\B}{\textbf}
\makeatletter
\newcommand{\clearsubcaptcounter}{\setcounter{sub\@captype}{0}}
\makeatother

\newcolumntype{C}[1]{>{\centering\arraybackslash}p{#1}}

\definecolor{revisioncolor}{rgb}{0.1,0.1,1}

\begin{document}

\maketitle
\thispagestyle{empty}
\pagestyle{empty}

\begin{abstract}
Computing collision-free trajectories is of prime importance for safe navigation. We present an approach for
computing the collision probability under Gaussian distributed motion and sensing uncertainty with the robot and static obstacle
shapes approximated as ellipsoids. The collision condition is formulated as the distance between ellipsoids and unlike
previous approaches we provide a method for computing the exact collision probability. Furthermore, we provide a tight
upper bound that can be computed much faster during online planning. Comparison to other state-of-the-art methods is also
provided. The proposed method is evaluated in simulation under varying configuration and number of obstacles.

\end{abstract}

\section{Introduction}
\label{sec:intro}
Safe and reliable navigation of robots is of vital importance while operating in real-world environments. To this
end, it is required to foresee collisions to compute collision free trajectories online. However, the robot state estimates
are uncertain due to motion and sensing errors. Similarly, environmental errors due to imperfect sensing, partial knowledge and estimation errors render obstacle location uncertain. As a result, rather than resorting to deterministic approaches that avoid collisions, collision avoidance algorithms need to incorporate such estimation uncertainties and compute the probability for encountering collisions.

Similarly to other existing approaches~\cite{dutoit2011IEEE,patil2012ICRA,park2018IEEE,axelrod2018IJRR,zhu2019RAL,thomas2021ISR}, in this work we model the uncertainties using Gaussian distributions. The robot and obstacle locations are thus parametrized as Gaussian probability distribution functions (pdfs). Given the robot and obstacle location, the probability of collision can be computed by marginalizing the joint distribution of the robot and the obstacle positions over the set of all robot and obstacle locations that lead to collisions. For example, if we assume the robot and the obstacle to be spheres, a collision occurs whenever the magnitude of the difference in robot and obstacle location is less than the sum of the radius of the spheres. The integral is thus marginalized over this collision constraint. However, there is no closed form solution to this marginalization. Most existing approaches~\cite{lambert2008ICCARV,dutoit2011IEEE,park2018IEEE,zhu2019RAL} therefore compute an approximation of the integral. These approximations can be overly conservative as the computed collision probabilities tend to be loose upper bounds. To circumvent the issue of non existence of a closed form solution, in this paper we show that the probability of collision is equivalent to the cumulative distribution function (cdf) of a quadratic form in random variables of the difference in robot and obstacle locations. We further compute a tighter bound for the probability of a quadratic form to be below a user-specified threshold. 

\noindent
\textit{Contributions:} In this paper, we present a novel approach to compute the exact collision probability when the uncertainties are modeled using Gaussian distributions. We formulate the collision constraint as the distance between \textit{L\"{o}wner-John ellipsoids} (Section~\ref{sec:coll_prob}) of the robot and the obstacles. This results in a quadratic form in random variables of the difference in robot and obstacle locations and we show that the required collision probability is the cdf of the quadratic form. An exact expression for the collision probability is thus obtained as a converging infinite series for the cdf. Furthermore, we derive a tighter upper bound with respect to previous works for fast approximation of the collision probability for 3D motion planning.

\section{Related Work}
\label{sec:related}
Uncertain environments are such that they often preclude the existence of guaranteed collision free trajectories~\cite{aoude2013AR}. Therefore, one can only provide probabilistic formal guarantees for safe navigation. Bounding volume approaches enlarge the robot and obstacle shapes, most often by their 3-$\sigma$ uncertainty ellipsoids or spheres. In~\cite{bry2011ICRA}, the uncertainty bound is used to enlarge the robot into a sphere. The robot shape is enlarged with the 3-$\sigma$ uncertainty ellipsoids in~\cite{kamel2017IROS}. Exact collision checking with enlarged bounding volumes is used in~\cite{park2012ICAPS}. Sigma hulls are used to compute the signed distance to the obstacles in~\cite{lee2013IROS}. Rectangular bounding box for both the robot and the obstacle are used to obtain an upper bound in~\cite{hardy2013TRO}. Patil \textit{et al.}~\cite{patil2012ICRA} truncate the Gaussian robot state distributions~\cite{johnson1994truncatedGaussian}, propagating collision free samples. Truncated Gaussian distributions is also employed in\cite{liu2014ICRA} to compute risk-aware and asymptotically optimal trajectories.  

Several approaches compute the collision probability by marginalizing the joint distribution between the robot and obstacle locations. The marginalization is computed over the set of robot and obstacle locations that satisfy the collision constraint. However, there is no closed form solution to such a formulation and  Monte Carlo (MC) techniques are employed~\cite{schmerling2017RSS}. In~\cite{lambert2008ICCARV}, the resulting double summation from the MC integration is approximated to a single summation. Another related MC approach and is presented in~\cite{janson2018ISRR}--- solves a deterministic motion planning problem with inflated obstacles, and then refines the inflation to compute a path that is exactly as safe as desired.  However, MC approaches tend to be computationally expensive and hard to model within an optimization framework. Assuming that the robot size is negligible, the marginalization can be approximated to obtain an expression in terms of the joint distribution multiplied by the volume occupied by the robot. Du Toit and Burdick~\cite{dutoit2011IEEE} compute the joint density at the robot center, whereas in~\cite{park2018IEEE} the maximum value (upper bound) is computed by evaluating the density at the surface of the robot.

Chance-constrained approaches find optimal sequence of control inputs subject to the constraint that the collision probability must be below a user-specified threshold~\cite{blackmore2011TRO}. Zhu \textit{et al.} compute an approximate upper bound linearizing the collision condition. They further employ chance constraints to compute bounded collision-free trajectories with dynamic obstacles.  An upper bound is computed in~\cite{sun2016ISRR} by employing Gaussian chance-constraints. They use Newtons's method to compute the point on the boundary of the obstacle that is closest to the robot configuration. In~\cite{aoude2013AR}, a Gaussian Process (GP) based technique is employed to learn motion patterns (a mapping from states to trajectory derivatives) to predict possible future obstacles trajectories. First-exist times for Brownian motions are extended to continuous nonlinear dynamics to compute collision probabilities in~\cite{frey2020RSS}. Axelrod \textit{et al.}~\cite{axelrod2018IJRR} focus exclusively on obstacle uncertainty and formalize a notion of \textit{shadows}, which are the geometric equivalent of confidence intervals for uncertain obstacles. The shadows fundamentally give rise to loose bounds but the computational complexity of bounding the collision probability is greatly reduced. Uncertain obstacles are modelled as polytopes with Gaussian-distributed faces in~\cite{shimanuki2018WAFR}. Planning a collision-free path in the presence of \textit{risk zones} is considered in~\cite{salzman2017ICAPS} by penalizing the time spent in these zones. Jasour \textit{et al.}\cite{jasour2019RSS} employ risk contours map that takes into account the risk
information (uncertainties in location, size and geometry of obstacles) to obtain safe paths with bounded risks. A related approach for randomly moving obstacles is introduced in~\cite{hakobyan2019RAL}. Collision avoidance in the context of human-robot interaction are presented in~\cite{bajcsy2019ICRA,fridovich2020IJRR}. Formal verification methods have also been used to construct safe plans~\cite{ding2013ICRA, sadigh2016RSS}. On the one hand, existing approximations tend to overestimate collision probabilities and could gauge all plans to be infeasible. On the other hand some approximations can be lower than the true collision probability values and can lead to synthesizing unsafe plans.

\section{Preliminaries}
Throughout this paper vectors will be assumed to be column vectors and will be denoted by bold lower case letters, $\textbf{x}$ and its components will be denoted by lower case letters. Transpose of $\textbf{x}$ will be denoted by $\textbf{x}^T$ and its Euclidean norm by $\norm{\textbf{x}} = \sqrt{\textbf{x}^T\textbf{x}}$. The mean of a random vector, $\EX(\B{x})$ will be denoted by $\bm{\mu}$, the corresponding covariance, $Cov(\textbf{x})$ by $\Sigma$. Matrices will be denoted by capital letters, $M$, with its trace denoted by $tr(M)$. By $vec(M)$ we will denote the vector formed used the columns of $M$. The identity matrix will be denoted by $I$ or $I_n$ (if dimension needs to be stressed). A diagonal matrix with diagonal elements $\lambda_1, \ldots, \lambda_n$ will be denoted by $diag(\lambda_1, \ldots, \lambda_n)$. Sets will be denoted using mathcal fonts, $\mathcal{S}$. Unless otherwise mentioned, subscripts on vectors/matrices will be used to denote time indexes. The notation $P(\cdot)$ will be used to denote the probability of an event and the pdf by $p(\cdot)$.
\begin{definition}
Given a positive definite $n \times n$ matrix A and a vector $\textbf{a} \in \mathbb{R}^n$, an n-dimensional ellipsoid $\mathcal{E}^n(A,\textbf{a})$ is defined as 
\begin{equation}
\mathcal{E}^n(A,\textbf{a}) = \{\B{x} \in \mathbb{R}^n | \left(\B{x}-\B{a}\right)^TA\left(\B{x}-\B{a}\right) \leq 1 \}
\end{equation}
\end{definition}
\noindent where the superscript $n$ will be avoided when there is no cause for confusion. At any time $k$, we denote the robot state by $\textbf{x}_k$, the acquired measurement from objects is denoted by $\textbf{z}_k$ and the applied control action is denoted as $\textbf{u}_k$. By \textit{objects} we refer to the landmarks and the obstacles in the environment. We make the following assumptions: (1) the uncertainties are modeled using Gaussian distributions, (2) the robot and obstacle geometry are known to us and are assumed to be non-deformable objects, (3) the collision probability at each time-step is treated to be independent with respect to the previous time-step. To describe the dynamics of the robot, we consider a standard motion model
\begin{equation}
\textbf{x}_{k+1} = f(\textbf{x}_k,\textbf{u}_k) + n_{k}\  ,  \ n_{k} \sim \mathcal{N}(0,R_{k})
\label{eq:odometry_model}
\end{equation}
\noindent where $n_k$ is the random unobservable noise, modeled as a zero mean Gaussian. Objects are detected through the robot's sensors and assuming known data association, the observation model can be written as  
\begin{equation}
\B{z}_k = h(\B{x}_k) + v_k \  ,  \ v_k \sim \mathcal{N}(0,Q_k)
\label{eq:measurement_model}
\end{equation}

\section{Collision Probability}
\label{sec:coll_prob}
For computing collision-free paths it is imperative that the distance between the robot and its environment is known. The geometry of robots and other objects in the environment can be expressed using polyhedrons or a combination of polyhedrons. It is known that polyhedrons can be approximated using suitable ellipsoids. In particular, in~\cite{john1948extremum} it was shown that for every convex polyhedron $\mathcal{P} \subseteq \mathbb{R}^n$, there exits a unique ellipsoid $\mathcal{E}$ of minimal volume that contains $\mathcal{P}$, called the \textit{L\"{o}wner-John ellipsoid} $J(\mathcal{P})$. Thus, there exists a unique minimum volume ellipsoid that encloses every $\mathcal{P}$ and it is sufficient to compute the distance between the L\"{o}wner-John ellipsoids of the robot and the obstacle, respectively. A convex optimization approach for computing L\"{o}wner-John ellipsoids is presented in~\cite{rimon1997JINT}. In this paper we assume that these ellipsoids are known to us. For an ellipsoid $\mathcal{E}(A,\textbf{a})$, $\B{a}$ is the center of the ellipsoid and $A$ determines its geometry. Thus our assumption of known ellipsoids mean that both these quantities are known at the initial time instant for robot and obstacles. Typically, $\B{a}$ is the center of the robot or the objects that the ellipsoid encloses. It is noteworthy that while planning for future control commands, the robot state (location and orientation) are often estimated using a motion model and by simulating possible future observations. As a result, the estimates for $\B{a}$ as well as the orientation are computed at each planning instance. Since we assume Gaussian distributed pdfs for state estimation, $\B{a}$ is typically taken to be mean of the location estimate and $A$ is obtained from the orientation estimate using an appropriate rotation matrix. In this paper we will only consider static obstacles and their L\"{o}wner-John ellipsoids remain constant throughout.
\subsection{Distance Between Two Ellipsoids}
For two ellipsoids $\mathcal{E}_1, \mathcal{E}_2$, we begin by computing a point $\B{x}^* \in \mathcal{E}_2$ such that the ellipsoid level surfaces surrounding $\mathcal{E}_1$ first touch $\mathcal{E}_2$ at $\B{x}^*$. The result is stated below without proof (see~\cite{rimon1997JINT} for proof).
\begin{lemma}
Given two ellipsoids $\mathcal{E}_1 = \mathcal{E}^n(B, \B{b})$ and $\mathcal{E}_2 = \mathcal{E}^n(C, \B{c})$, the point $\B{x}^* \in \mathcal{E}_2$ at which the ellipsoid level surfaces surrounding $\mathcal{E}_1$ first touch the ellipsoid $\mathcal{E}_2$ is given by
\begin{equation}
\B{x}^* = \B{b} + \lambda_0(M')B^{-1/2}\left[\lambda_0(M')I - \tilde{C}\right]^{-1}\bar{c}
\end{equation}
\noindent where $\lambda_0(M')$ is the minimal eigenvalue of a $2n \times 2n$ matrix $M'$ given by
\begin{equation}
M' = \begin{bmatrix}
\tilde{C} & -I  \\
-\tilde{c}\tilde{c}^T & \tilde{C}
\end{bmatrix} 
\label{eq:mprime}
\end{equation}
\noindent with $\tilde{C} = \bar{C}^{-1}$ and $\tilde{c} = \bar{C}^{-1/2}\bar{c}$, where $\bar{C} = B^{-1/2}CB^{-1/2}$ and $\bar{c} = B^{1/2}(\B{c} - \B{b})$.
\label{lemma2}
\end{lemma}
For two ellipsoids $\mathcal{E}_1 =  \mathcal{E}^n(B, \B{b})$ and $\mathcal{E}_2 = \mathcal{E}^n(C, \B{c})$, a collision between them occurs if $\mathcal{E}_1 \cap \mathcal{E}_2 \neq \{\emptyset\}$. We will denote this collision condition by
\begin{equation}
\mathcal{C}_{\B{b},\B{c}} \doteq \left \{\mathcal{E}_1, \mathcal{E}_2 | \mathcal{E}_1 \cap \mathcal{E}_2 \neq \{\emptyset\}\right \}
\end{equation}
Using Lemma~\ref{lemma2}, for two ellipsoids $\mathcal{E}^n(B, \B{b})$, $\mathcal{E}^n(C, \B{c})$, we can compute the point $\B{x}^* \in \mathcal{E}_2$ at which the ellipsoid level surfaces surrounding $\mathcal{E}_1$ first touch $\mathcal{E}_2$. Now, suppose that the two ellipsoids touch or intersect each other, then $\B{x}^*$ satisfies the equation of $\mathcal{E}_1$. Thus we have
\begin{equation}
(\B{x}^* - \B{b})^TB(\B{x}^* - \B{b}) \leq 1
\end{equation}
Substituting for the value of $\B{x}^*$ from Lemma~\ref{lemma2}, expanding and rearranging, it follows that
\begin{equation}
\B{y}^TD^TBD\B{y} \leq 1/\lambda_0^2(M')
\label{eq:coll_con}
\end{equation}
where $\B{y} = \B{c} - \B{b}$ and $D = B^{-1/2}(\lambda_0(M')I - \tilde{C})^{-1}B^{1/2}$ and $M'$ is a $2n \times 2n $ matrix as defined in~(\ref{eq:mprime}).
\subsection{Collision Condition}
\noindent We first define a quadratic form in random variables and later show that the collision condition can be written in terms of the quadratic form of the difference in robot and obstacle locations. 
\begin{definition}
For a random vector $\mathbf{x} = \left(x_1, \ldots,x_n \right)^T$, the quadratic form in the random variables $x_1, \ldots,x_n $ associated with an $n \times n$ symmetric matrix $A = (a_{ij})$ is
\end{definition}
\vspace{-0.6cm}
\begin{equation}
Q(\mathbf{x},A) = Q(x_1, \ldots,x_n) = \mathbf{x}^TA\mathbf{x} = \sum_{i=1}^{n}\sum_{j=1}^n a_{ij}x_ix_j
\end{equation}
Let us now consider a random vector $\B{x}$ with mean value $\bm{\mu}$ and a positive definite covariance matrix $\Sigma > 0$. Then for $\mathbf{y} = \Sigma^{-\frac{1}{2}}\mathbf{x}$, where $\Sigma^{\frac{1}{2}}$ is the symmetric square root, we have $\EX(\B{y}) = \Sigma^{-\frac{1}{2}}\bm{\mu}$ and $Cov(\B{y})=I$. Let us define a random vector $\mathbf{z} = \left(\mathbf{y} - \Sigma^{-\frac{1}{2}}\bm{\mu}\right)$ such that $\EX(\B{z}) = 0$ and $Cov(\B{z}) = I$. Then
\begin{multline}
Q(\mathbf{x},A)  = \mathbf{x}^TA\mathbf{x} = \mathbf{y}^T \Sigma^{\frac{1}{2}} A \Sigma^{\frac{1}{2}} \mathbf{y} \\
=\left(\mathbf{z} + \Sigma^{-\frac{1}{2}}\bm{\mu}\right)^T\Sigma^{\frac{1}{2}}A\Sigma^{\frac{1}{2}}\left(\mathbf{z} + \Sigma^{-\frac{1}{2}}\bm{\mu}\right)
\end{multline}
Let $P$ be an orthogonal matrix, that is, $PP^T = I$ which diagonalizes $\Sigma^{\frac{1}{2}}A\Sigma^{\frac{1}{2}}$, then $P^T\Sigma^{\frac{1}{2}}A\Sigma^{\frac{1}{2}}P = \textrm{diag}\left(\lambda_1,\ldots,\lambda_n\right)$, where $\lambda_1,\ldots,\lambda_n$ are the eigenvalues of $\Sigma^{\frac{1}{2}}A\Sigma^{\frac{1}{2}}$. The quadratic form can now be written as
\begin{equation}
\begin{split}
Q(\mathbf{x},A) & = \left(\mathbf{z} + \Sigma^{-\frac{1}{2}}\bm{\mu}\right)^T\Sigma^{\frac{1}{2}}A\Sigma^{\frac{1}{2}}\left(\mathbf{z} + \Sigma^{-\frac{1}{2}}\bm{\mu}\right)\\
& = \left(\mathbf{u} + \mathbf{b}\right)^T\textrm{diag}\left(\lambda_1,\ldots,\lambda_n\right)\left(\mathbf{u} + \mathbf{b}\right)
\end{split}
\label{eq:quad_form}
\end{equation}
\noindent where $\mathbf{u} = P^T \mathbf{z} = (u_1,\ldots,u_n)^T$ and $\mathbf{b} = P^T \Sigma^{-\frac{1}{2}}\bm{\mu} = (b_1,\ldots,b_n)^T$. The expression in (\ref{eq:quad_form}) can be written concisely,
\begin{equation}
Q(\mathbf{x}) = \mathbf{x}^TA\mathbf{x} = \sum_{i=1}^n \lambda_i (u_i + b_i)^2
\end{equation}
We recall from~(\ref{eq:coll_con}) that the collision condition is given by
\begin{equation}
\B{y}^TD^TBD\B{y} = \B{y}^TA\B{y} \leq 1/\lambda_0^2(M')
\label{eq:coll_condition}
\end{equation}
\noindent where $A = D^TBD$. In motion planning, the collisions correspond to robot-obstacle or robot-robot collisions. While planning under motion and sensing uncertainty, the robot and obstacle states can only be estimated in probabilistic terms. This renders $\B{b}$, $\B{c}$ as random vectors.  As discussed before, in this paper we assume them to have Gaussian pdfs. As a result, the difference vector $\B{y} = \B{c} - \B{b}$ is also a Gaussian random vector. Since $B$ and $C$ are positive definite, the matrix $A$ is symmetric and therefore $\B{y}^TA\B{y}$ is in a quadratic form. Thus the collision probability can be written as
\begin{equation}
P\left(\B{y}^TA\B{y} \leq 1/\lambda_0^2(M')\right)
\label{eq:collision_prob}
\end{equation}
Let $\B{v}= \B{y}^TA\B{y}$, then 
\begin{equation}
P\left(\B{v} \leq 1/\lambda_0^2(M')\right) = F_{\B{v}}\left(1/\lambda_0^2(M')\right)
\end{equation}
\noindent where $F_{{\B{v}}}$ is the cdf of $\B{v}$. For convenience, we state the theorem that computes an exact expression of $F_{{\B{v}}}$ for quadratic forms. The proof may be found in~\cite{provost1992book}.
\begin{theorem}
The cdf of $Q(\mathbf{y},A) = \textnormal{\B{v}} = \mathbf{y}^TA\mathbf{y}$ with $A = A^T > 0, \mathbf{y} \sim \mathcal{N}(\bm{\mu},\Sigma), 
\Sigma > 0$ is 
\begin{equation}
F_{\textnormal{\B{v}}}(v) = P(\textnormal{\B{v}}\leq v) = \sum_{k=0}^{\infty}(-1)^k c_k \frac{v^{\frac{n}{2} + k}}{ \Gamma\left(\frac{n}{2} +k +1\right)}
\label{eq:collision_probability}
\end{equation}
and its pdf is given by 
\begin{equation}
p_{\textnormal{\B{v}}}(v) = P(\textnormal{\B{v}} = v) = \sum_{k=0}^{\infty}(-1)^k c_k \frac{v^{\frac{n}{2} + k -1}}{ \Gamma\left(\frac{n}{2} +k \right)}
\label{eq:pdf}
\end{equation}
\label{theoremcdf}
\end{theorem}
\noindent where $\Gamma$ denotes the gamma function and $c_0 = \textrm{exp}(-\frac{1}{2}\sum\limits_{i=1}^{n}b_i^2)\prod_{i=1}^n\left(2\lambda_i\right)^{-\frac{1}{2}}$, $c_k = \frac{1}{k}\sum\limits_{i=0}^{k-1}d_{k-i}c_i$ and $d_k = \frac{1}{2}\sum\limits_{i=1}^n \left(1-kb_i^2\right)\left(2\lambda_i\right)^{-k}$. Thus computing the cdf as elucidated in Theorem~\ref{theoremcdf} gives the exact value of collision probability. The cdf is computed as an infinite series; a proof of convergence, an expression for the truncation error and the computational complexity can be found in~\cite{thomas2021ISR}. In our experience, suitable convergence is often obtained within the first few terms and hence can be used for online planning.


\noindent \textit{Special case:} We provide the theorem without proof (see~\cite{provost1992book}) that states the necessary and sufficient conditions for a quadratic form of a Gaussian random variable to be distributed as a noncentral chi-square variate. The exact collision probability can then be computed much faster using the cdf of the chi-square distribution by table look-up.
\begin{theorem}
Let $\B{x} \sim \mathcal{N}(\bm{\mu},\Sigma)$, $\Sigma > 0$. Then the necessary and sufficient conditions for $\B{x}^TA\B{x} \sim \chi_{r}^2(\Delta^2)$, $A = A^T$ and $\Delta^2 = \bm{\mu}^TA\bm{\mu}$ are: $tr(A\Sigma) = r$ and $A\Sigma A = A$.
\label{theorem_special}
\end{theorem}
We note that $\Sigma$ depends on the robot and obstacle location estimates, yet $A = A^T$ is satisfied by the design of the ellipsoids. Although inspecting the conditions at each time instant might be tedious during online planning, the theorem can be used to determine the accurate cdf of $\B{v}$ during offline planning to compute initial collision free trajectories.
\subsection{Tighter Bound for Collision Probability}
During online motion planning it often suffices to compute fast approximate upper bounds for the collision probability. In this section, we thus derive a tighter bound for the same.
\begin{lemma}
Let $\mathbf{v}$ be a random variable that is never larger than $\beta$. Then, for all $\alpha < \beta$
\begin{equation}
P(\mathbf{v} \leq \alpha) \leq \frac{\beta - \EX(\mathbf{v})}{\beta - \alpha}
\end{equation}
\label{lemma_markov}
\end{lemma}
\vspace{-0.7cm}
\begin{proof}
From Markov's inequality, for all $\alpha >0$, we have
\begin{equation*}
P(\mathbf{v} \geq \alpha) \leq \frac{\EX(\mathbf{v})}{\alpha}
\end{equation*}
\noindent Let us now define $\tilde{\mathbf{v}} = \beta - \mathbf{v}$ such that $P({\mathbf{v}} \leq \beta) = 1$. Thus, we have
\begin{equation}
 P({\mathbf{v}} \leq \alpha) = P(\beta - \tilde{\mathbf{v}} \leq \alpha ) = P( \tilde{\mathbf{v}} \geq \beta - \alpha)
\end{equation} 
\noindent Applying Markov's inequality to $\tilde{\mathbf{v}}$, we get
\begin{equation}
 P( \tilde{\mathbf{v}} \geq \beta - \alpha) \leq \frac{\EX(\tilde{\mathbf{v}})}{\beta - \alpha} = \frac{\beta - \EX(\mathbf{v})}{\beta - \alpha}
\end{equation}
\vspace{-0.3cm}
\end{proof}

\noindent Using the above lemma, an upper bound for the collision probability is obtained as
\begin{equation}
P\left(\B{y}^TA\B{y} \leq 1/\lambda_0^2(M')\right) \leq \frac{\beta - \EX(\B{y}^TA\B{y})}{\beta - 1/\lambda_0^2(M')}
\label{eq:cp_upper_bound}
\end{equation}
The expectation in~(\ref{eq:cp_upper_bound}) is computed using the fact that for $\B{x} \sim \mathcal{N}(\bm{\mu},\Sigma)$, $\EX(\mathbf{x}^TA\mathbf{x}) = tr(A\Sigma) + \bm{\mu}^TA\bm{\mu}$.

\noindent \textit{Computing $\beta$:} For a symmetric matrix $A \in \mathcal{R}^{n \times n}$ and a random vector $\B{x}$, we have $\B{x}^T A \B{x} \leq \lambda_{max}\norm{\B{x}}^2$, where $\lambda_{max}$ is the maximum eigenvalue of $A$. Thus if $\B{x}$ lies within a unit sphere, the quadratic form is bounded by $\lambda_{max}$. Note that in our case $\B{x}$ is a random vector which is obtained from the difference between robot and obstacle locations. Since the environment in which the robot operates is bounded, without any loss of generality we can assume that $\B{x}$ lies within a sphere of certain radius $\kappa$\footnote{This means truncating $\B{x}$ at the boundary of the workspace. However, we assume the workspace to be large enough so that the truncated $\B{x}$ is approximately Gaussian distributed.}. Thus, $\beta = \max(\B{x}^T A \B{x}) = \lambda_{max} \kappa^2$. Since $\B{x}^T A \B{x}$ is a real number, we may thus write $\B{x} \in \left[\EX(\mathbf{x}^TA\mathbf{x})- \eta\sqrt{Var(\mathbf{x}^TA\mathbf{x})}, \EX(\mathbf{x}^TA\mathbf{x})+\eta\sqrt{Var(\mathbf{x}^TA\mathbf{x})}\right]$, where $\eta$ $\in$ $\mathbb{R}^+$. Based on several experimental analysis, $\eta$$=$$1$ provides a tighter bound and works well in practice in all our scenarios. We thus use $\beta = \EX(\mathbf{x}^TA\mathbf{x}) + \sqrt{Var(\mathbf{x}^TA\mathbf{x})}$.
\begin{definition}
A robot configuration $\B{x}_k$ is an $\epsilon-$safe configuration with respect to an obstacle configuration $\B{s}_k$, if the probability of collision is such that $P\left(\mathcal{C}_{\B{x}_k,\B{s}_k}\right) \leq \epsilon$.
\end{definition}
Let us consider the case where the ellipsoids enclose a robot currently at location $\B{x}_k$ and an obstacle at location $\B{s}_k$, respectively. Since the position of robots and obstacles are Gaussian distributed random variables, there exists a nonzero probability for both the robot and the obstacle to be found anywhere within the environment. Therefore, trajectories are computed such that the probability of collision with the obstacles is less than a specified bound. We thus look for robot positions $\B{x}_k$ such that the probability of collision is at most $\epsilon$, that is, $P\left(\mathcal{C}_{\B{x}_k,\B{s}_k}\right) \leq \epsilon$. From (\ref{eq:collision_prob}), we have
\begin{equation}
P\left(\B{y}^TA\B{y} \leq 1/\lambda_0^2(M')\right) \leq \epsilon
\label{eq:cp_bounded}
\end{equation}
\noindent Using Lemma~\ref{lemma_markov}, the bounded collision constraint is obtained 
\begin{equation}
\beta - \EX(\B{y}^TA\B{y})  \leq \epsilon\left(\beta -  1/\lambda_0^2(M')\right)
\label{eq:upper_bound}
\end{equation}
\subsection{Comparison to Other Methods}
We provide a comparison with several state-of-the-art methods using a robot and a close-by obstacle. The robot mean is located at $(0.95,0.95,0)$ with its L\"{o}wner-John ellipsoid semi-principle axes $(0.18,0.18,0.22)$ m and covariance $diag(0.41, 0.41,0.21)$ $\textrm{m}^2$. The obstacle is located at the origin with semi-principle axes $(0.6,0.6,1.2)$ m. The collision probability values can be seen in Table~\ref{table1}. Most approaches approximate an integral of the joint distribution between the robot and the obstacle to compute the collision probability. Given the current robot state $\B{x}_k$ and the obstacle state $\B{s}_k$, the collision probability is given by
\begin{equation}
P\left(\mathcal{C}_{\B{x}_k,\B{s}_k}\right) = \int_{\B{x}_k} \int_{\B{s}_k} I_c(\B{x}_k,\B{s}_k)p(\B{x}_k,\B{s}_k)
\label{eq:numerical}
\end{equation}
where $I_c$ is an indicator function defined as
\begin{equation}
   I_c(\B{x}_k,\B{s}_k)= 
   \begin{cases}
     1 \ &\text{if} \ \mathcal{R} \cap \mathcal{S} \neq \{\phi\} \\
     0 \ &\text{otherwise}.
   \end{cases}
\end{equation}
\noindent and $p(\B{x}_k,\B{s}_k)$ is the joint distribution of the robot and the obstacle. The numeric integral of~(\ref{eq:numerical}) gives the exact value and we compute the same to validate our approach. The numerical integral gives a collision probability value of 0.0568. If we define $\epsilon$$=$$0.09$, this configuration is $0.09-$safe or is feasible. Our method which computes the exact value using the cdf of the quadratic form gave a value of 0.0572. The value computed using the upper bound~(\ref{eq:cp_upper_bound}) is very
close to the actual value as seen from Table~\ref{table1}. The double summation of numerical integration is approximated to a single summation in~\cite{lambert2008ICCARV} and gives a feasible result. The approaches in~\cite{thomas2021ISR,thomas2020IRIM} also give feasible configurations. However, the values computed using~\cite{thomas2021ISR,lambert2008ICCARV,thomas2020IRIM} are greater
than the value computed using our upper bound method. Other approaches compute loose bounds and hence determine the configuration infeasible. Our approach thus computes a tighter bound.
\begin{table}
\small\sf\centering
 \caption{Comparison of collision probability methods}
\scalebox{0.68}{
\begin{tabular}{ |c|c|c|c| } 
 \hline
 Methods & Collision  & Computation  & Feasible ? \\
 & probability & time (s) & \\
 \hline 
 Numerical integral & 0.0568 & 4.3619 $\pm$ 1.1784 & Yes \\ 
  \hline 
   Approximate Numerical integral~\cite{lambert2008ICCARV} & 0.0773 & 1.7932 $\pm$ 0.1927 & Yes \\ 
  \hline 
 Bounding volume~\cite{park2012ICAPS,kamel2017IROS} & 1 & 0.0003 $\pm$ 0.0016  & No\\ 
 \hline
 Center point approximation~\cite{dutoit2011IEEE} & 0.1027  & 0.0004 $\pm$ 0.0001 & No \\
  \hline
 Maximum probability approximation~\cite{park2018IEEE}  & 0.7168  & 0.2288 $\pm$ 0.1626 & No  \\
  \hline
 Chance constraint~\cite{zhu2019RAL} & 0.1894 & 0.0013 $\pm$ 0.0000 & No\\
  \hline
 Rectangular bounding box~\cite{hardy2013TRO} & 0.1582 & 0.0056 $\pm$ 0.0006 & No\\
  \hline
 Sphere approximation~\cite{thomas2021ISR}  & 0.0898 & 0.0198 $\pm$ 0.0309 & Yes\\
  \hline
 Our approach-- exact & 0.0572 & 0.0044 $\pm$ 0.0042 & Yes\\
  \hline
 Our approach-- upper bound & 0.0660 & 0.0009 $\pm$ 0.0003 & Yes\\
 \hline
\end{tabular}}
 \label{table1}
\end{table}

\section{Planning under Uncertainty}
\subsection{Belief Dynamics}
We consider a Gaussian parametrization for the probability distribution over the robot/obstacle state which is known as the \textit{belief} state. Belief Space Planning (BSP) has been researched extensively in the past and a comprehensive treatment can be found in~\cite{prentice2009IJRR, van_den_berg2012IJRR, agha_mohammadi2014IJRR, pathak2017ICRA}. The belief state $\B{b}(\B{x}_k) \sim \mathcal{N}(\bm{\mu}_k, \Sigma_k)$ is a Gaussian distribution with mean $\bm{\mu}_k$ and covariance $\Sigma_k$ we denote the belief state by a vector
\begin{equation}
\B{b}(\B{x}_k) = \left[ \bm{\mu}_k^T, \B{s}_k^T \right]^T
\end{equation}
\noindent where $\B{s}_k^T = [s_{k_1}^T, \ldots, s_{k_n}^T]$ is vector composed of the $n$ columns of $\Sigma_k$. Equivalently, $\B{s}^T$ will also be denoted by $vec(\Sigma_k)$  To compute the belief dynamics, we use the technique of Extended Kalman Filter (EKF)
\begin{equation}
\bm{\mu}_{k} = \bm{\bar{\mu}}_k + K_k\left(\B{z}_k - h(\bm{\bar{\mu}}_k)\right)
\label{eq:stochastic}
\end{equation}
\vspace{-0.5cm}
\begin{equation}
\Sigma_k = \left(I - K_kH_k\right)\bar{\Sigma}_{k} 
\end{equation}
\noindent where $\bm{\bar{\mu}}_k = f(\bm{\mu}_{k-1}, \B{u}_{k-1})$, $\bar{\Sigma}_{k} = F_{k-1} \Sigma_{k-1} F_{k-1}^T + R_{k-1}$ and $ K_k   = \bar{\Sigma}_{k} H_k^T\left(H_k \bar{\Sigma}_{k}  H_k^T + Q_k\right)^{-1}$, with $H_k$ being the Jacobian of $h(\cdot)$ with respect to $\B{x}_k$ and $F_{k-1}$ is the Jacobian of $f(\cdot)$ with respect to $\B{x}_{k-1}$. The second term in~(\ref{eq:stochastic}) depends on the measurement $\B{z}_k$ and is often referred to as the \textit{innovations process}. Since future observations are unknown at the planning time, the innovations process is stochastic. As a result the belief state dynamics is stochastic in nature. 
\begin{theorem}
The innovations process is a zero-mean Gaussian white noise sequence with
\label{th:innnovation}
\begin{equation}
\EX\left( \left[\B{z}_k - h(\bm{\bar{\mu}}_k)\right] \left[\B{z}_k - h(\bm{\bar{\mu}}_k)\right]^T \right) = H_k \bar{\Sigma}_{k}  H_k^T + Q_k
\label{eq:innovations}
\end{equation}
\end{theorem}
\begin{proof}
The proof may be found in~\cite{mendel1995estimation}, page 234.
\end{proof}

\noindent From Theorem~\ref{th:innnovation}, the stochastic belief state dynamics can be written as
\begin{equation}
\resizebox{.89\hsize}{!}{$
\B{b}(\B{x}_k)  = g\left( \B{b}(\B{x}_{k-1}), \B{u}_{k-1} \right) + W\left( \B{b}(\B{x}_{k-1}), \B{u}_{k-1} \right)w_{k-1}$}
\end{equation}

\noindent where
\begin{equation}
\resizebox{.89\hsize}{!}{$
g\left( \B{b}(\B{x}_{k-1}), \B{u}_{k-1} \right) = \begin{bmatrix}
\bm{\bar{\mu}}_k \\
\B{s}_k^T
\end{bmatrix} =
\begin{bmatrix}
f(\bm{\mu}_{k-1}, \B{u}_{k-1}) \\
vec\left(\left(I - K_kH_k\right)\bar{\Sigma}_{k}\right)
\end{bmatrix}$}
\end{equation}
\begin{equation}
W\left( \B{b}(\B{x}_{k-1}), \B{u}_{k-1} \right) = \begin{bmatrix}
K_k\\
\textbf{0}
\end{bmatrix}
\end{equation}
\begin{equation}
w_{k-1} \sim \mathcal{N}( \textbf{0}, H_k \bar{\Sigma}_{k}  H_k^T + Q_k)
\end{equation}

\noindent Thus the innovation term $K_k(\B{z}_k - h(\bm{\bar{\mu}}_k))$ is distributed according to 
\begin{equation}
\mathcal{N}(\textbf{0}, K_k(H_k \bar{\Sigma}_{k}  H_k^T + Q_k)K_k^T)
\end{equation}
The maximum likelihood observation at time $k$ is $\B{z}_k = h(\bar{\bm{\mu}}_k)$, which reduces the innovation term in~(\ref{eq:stochastic}) to zero and thus eliminating the stochasticity from the belief state dynamics. The assumption of maximum likelihood observation was first relaxed in~\cite{van_den_berg2012IJRR}. However, a first-order approximation is assumed rendering the innovation term to be distributed according to $\mathcal{N}(\textbf{0}, K_kH_k\bar{\Sigma}_{k})$. Other approaches that relax this assumption either simulate future measurements or treat them as random variables~\cite{indelman2015IJRR,pathak2018IJRR, thomas2019ISRR}. 
\subsection{Objective Function}
We formulate the collision avoidance problem in BSP as an optimization problem. At each time instant $k$, the robot plans for $L$ look-ahead steps minimizing an objective function $J_k = \underset{\B{z}_{k+1:k+L}}{ \EX} \left[ \sum\limits_{l=0}^{L-1} c_l(\B{b}(\B{x}_{k+l}),\B{u}_{k+l}) + c_L(\B{b}(\B{x}_{k+L}))\right]$, subject to certain constraints, with $c_l$ being the cost for each look-ahead step and $c_L$ the terminal cost. Since future observations are not available at planning time and are stochastic, the expectation is taken to account for all possible future observations. At each time step, the robot is required to minimize its control usage and proceed towards the goal $\B{x}^g$ avoiding collisions. As a result, we have the following immediate and terminal costs
\begin{eqnarray}
c_l(\B{b}(\B{x}_{k+l}),\B{u}_{k+l}) = \norm{\xi(\B{u}_{k+l})}^2_{M_u} \\
 c_L(\B{b}(\B{x}_{k+L})) = \norm{\B{x}_{k+L} - \B{x}^g}^2_{M_g} 
 \end{eqnarray}
\noindent where $\norm{x}_S = \sqrt{x^TSx}$ is the Mahalanobis norm, $M_u, M_g$ are weight matrices and $\xi(\cdot)$ is a function that quantifies control usage. $J_k$ can now be explicitly written as
\begin{equation}
\begin{split}
J_k & = \underset{\B{z}_{k+1:k+L}}{ \EX} \left[ \sum_{l=0}^{L-1}  \norm{\xi(\B{u}_{k+l})}^2_{M_u} +  \norm{\B{x}_{k+L} - \B{x}^g}^2_{M_g} \right]\\
  & = \sum_{l=0}^{L-1}  \norm{\xi(\B{u}_{k+l})}^2_{M_u} +  \underset{\B{z}_{k+L}}{ \EX} \left[ \norm{\B{x}_{k+L} - \B{x}^g}^2_{M_g} \right]
   \end{split}
 \label{eq:objective_fn}
\end{equation}  
\noindent The expectation is discarded from the first term as it does not depend on the future observations. Let us now proceed by evaluating the term with expectation. Using~(\ref{eq:stochastic}), we have
\begin{multline}
\resizebox{.96\hsize}{!}{$
\norm{\B{x}_{k+L} - \B{x}^g}^2_{M_g} = \norm{\bm{\bar{\mu}}_{k+L} + K_{k+L}\left(\B{z}_{k+L} - h(\bm{\bar{\mu}}_k)\right) - \B{x}^g}^2_{M_g} $}\\
= \norm{\bm{\bar{\mu}}_{k+L} - \B{x}^g + K_{k+L}\left(\B{z}_{k+L} - h(\bm{\bar{\mu}}_k)\right) }^2_{M_g}\\
= \norm{\bm{\bar{\mu}}_{k+L} - \B{x}^g}^2_{M_g} + \norm{K_{k+L}\left(\B{z}_{k+L} - h(\bm{\bar{\mu}}_k)\right)}^2_{M_g} \\ 
+\left(\bm{\bar{\mu}}_{k+L} - \B{x}^g\right)^TM_g\left(K_{k+L}\left(\B{z}_{k+L} - h(\bm{\bar{\mu}}_k)\right)\right)\\
 + \left(K_{k+L}\left(\B{z}_{k+L} - h(\bm{\bar{\mu}}_k)\right)\right)^TM_g\left(\bm{\bar{\mu}}_{k+L} - \B{x}^g\right)
\label{eq:norm_expand}
\end{multline}
\noindent For any random vector $\B{y}$ and a matrix $A$ of appropriate dimension, we have $\EX \left[ \B{y}^TA\B{y}\right] = tr\left( A \textrm{Var}(\B{y})\right) + \EX[\B{y}]^TA\EX[\B{y}]$. Exploiting this and the fact that $\EX \left[\left(\B{z}_{k+L} - h(\bm{\bar{\mu}}_k)\right) \right] = \B{0}$, the expression in~(\ref{eq:norm_expand}) simplifies to
\begin{multline}
\underset{\B{z}_{k+L}}{ \EX} \left[ \norm{\B{x}_{k+L} - \B{x}^g}^2_{M_g} \right]
=\\ \norm{\bm{\bar{\mu}}_{k+L} - \B{x}^g}^2_{M_g} +  tr\left( K_{k+L} \textrm{Var}\left(\B{z}_{k+L} - h(\bm{\bar{\mu}}_k)\right)\right)
\end{multline}
\noindent where the expression for the variance is obtained from~(\ref{eq:innovations}). 
The optimization problem can be formally stated now as
\vspace{-0.5cm}
\begin{mini!}|s|[1]
{ \B{b}_{k:k+L-1},  \B{u}_{k:k+L-1} }{J_k}{\label{eq:cost_fn}}{}
  \addConstraint{\B{b}(\B{x}_{k+l})}{=g( \B{b}(\B{x}_{k+l-1}), \B{u}_{k+l-1})}{\nonumber} 
  \addConstraint{+W( \B{b}(\B{x}_{k+l-1}), \B{u}_{k+l-1} )w_{k+l-1}}{}
  \addConstraint{\B{u}_{k+l}}{\in \B{U}}{\label{eq:controls}}
  \addConstraint{P\left(\mathcal{C}_{\B{x}_{k+l},\B{s}_{k+l}}\right)}{\leq \epsilon}{\label{eq:p_coll}}
\end{mini!}
\noindent where ~(\ref{eq:controls}) constraints the control input to lie within the feasible set of control inputs and~(\ref{eq:p_coll}) enforces $\epsilon-$safe configurations. We will now derive the expression for~(\ref{eq:p_coll}). Let the L\"{o}wner-John ellipsoids of the robot and the obstacle be $\mathcal{E}(X, \B{x}_{k+l})$ and $\mathcal{E}(S, \B{s}_{k+l})$, respectively. Using~(\ref{eq:upper_bound}), the constraint ~(\ref{eq:p_coll}) becomes 
\begin{multline}
\beta - \EX \left[ \left( \B{s}_{k+l} - \B{x}_{k+l} \right)^TA_{k+l}\left( \B{s}_{k+l} - \B{x}_{k+l} \right)\right]\\  \leq \epsilon\left(\beta -  1/\lambda_0^2(M'_{k+l})\right)
\label{eq:coll_expectation}
\end{multline} 
\noindent where $M'_{k+l}$ is computed using~(\ref{eq:mprime}) and $\lambda_0(M'_{k+l})$ is the minimal eigenvalue of $M'_{k+l}$. 


\section{Results}
\begin{figure}[t]
\vspace{-0.2cm}
  \subfloat{\includegraphics[scale=0.193]{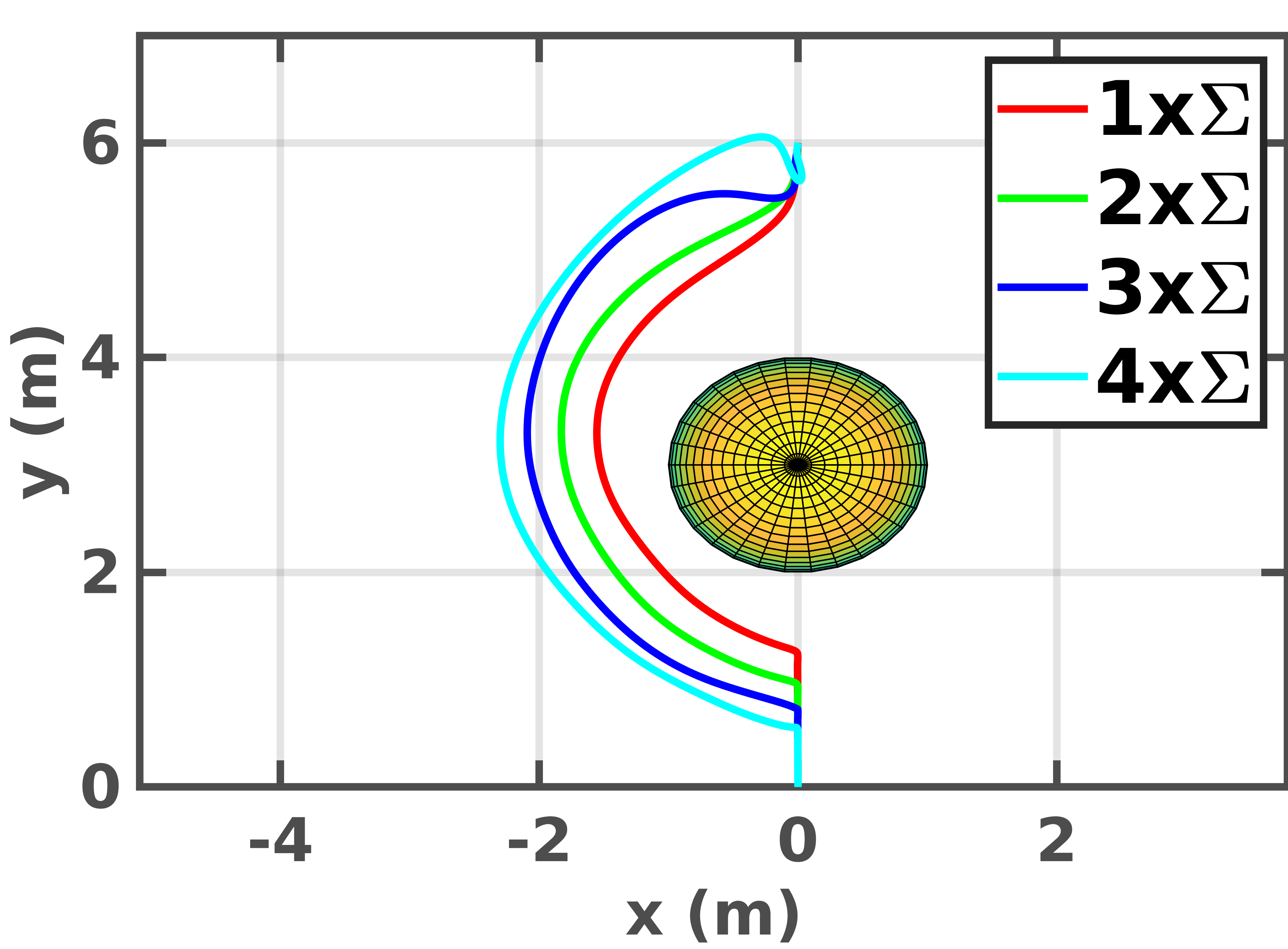}\label{fig:cp1}}\hfill
  \subfloat{\includegraphics[scale=0.193]{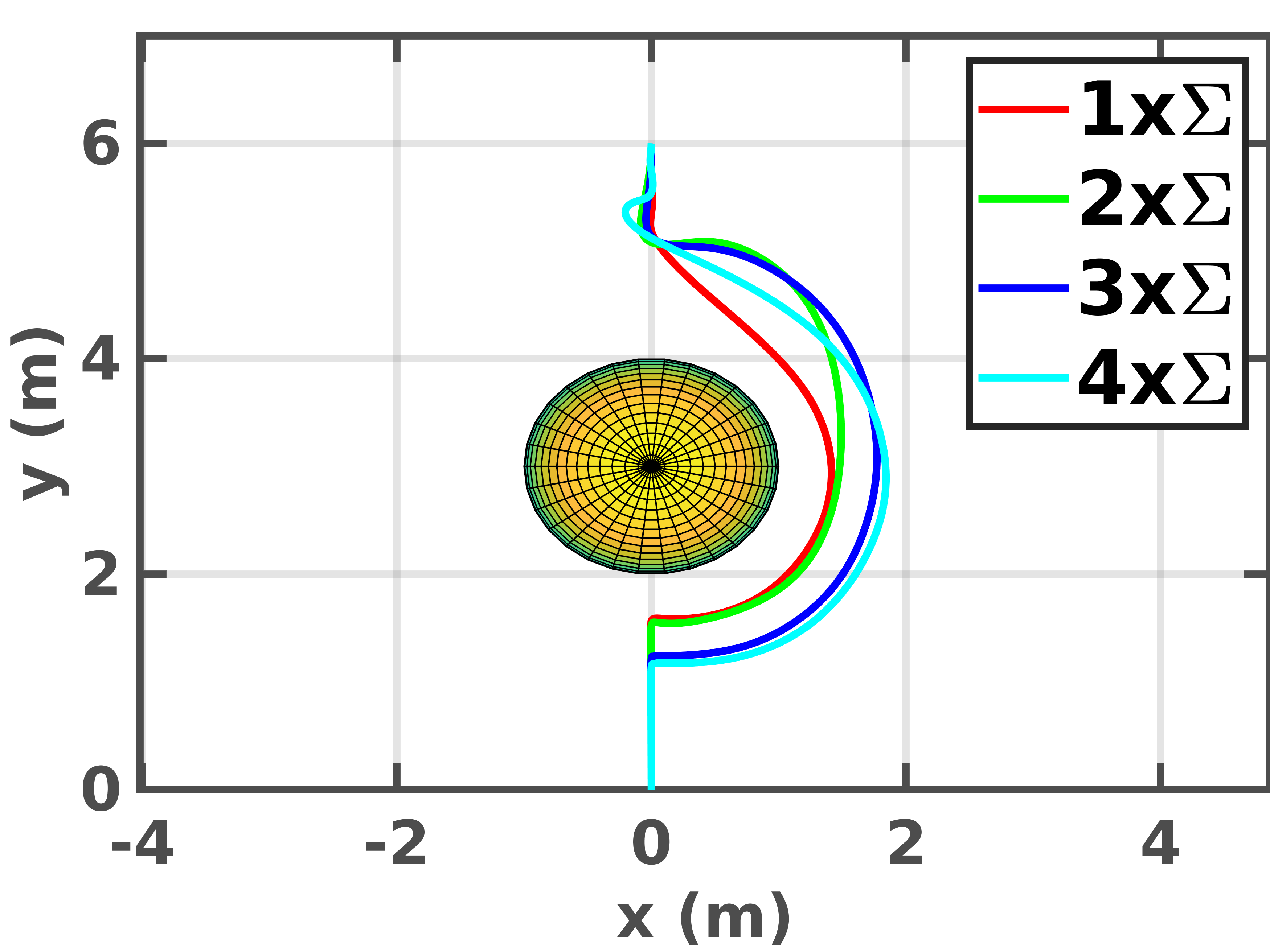}\label{fig:cp2}}\hfill
  \subfloat{\includegraphics[scale=0.183]{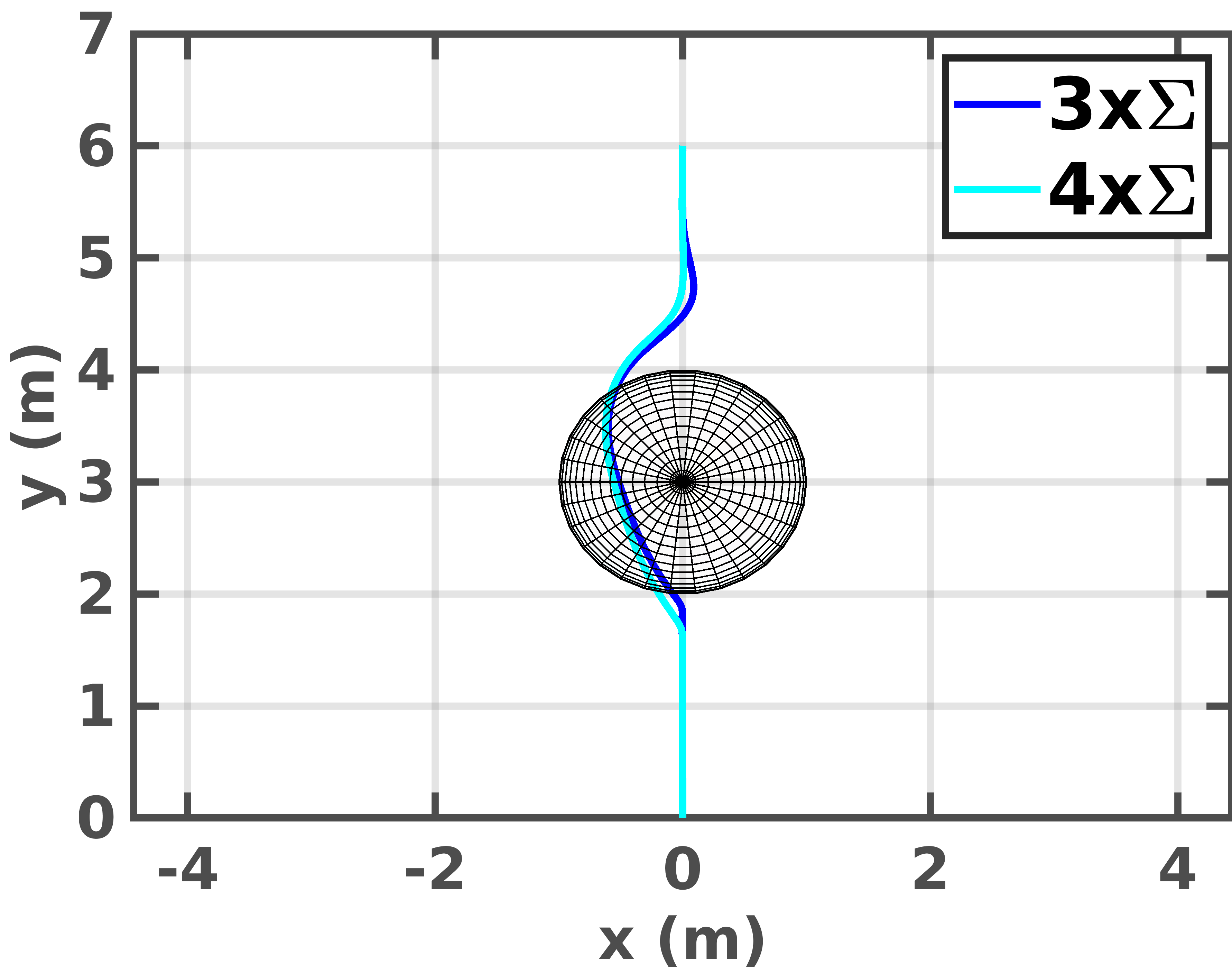}\label{fig:cp3}}
  \vspace{-0.3cm}
  \clearsubcaptcounter
  \subfloat[Bounding volume]{\includegraphics[scale=0.19]{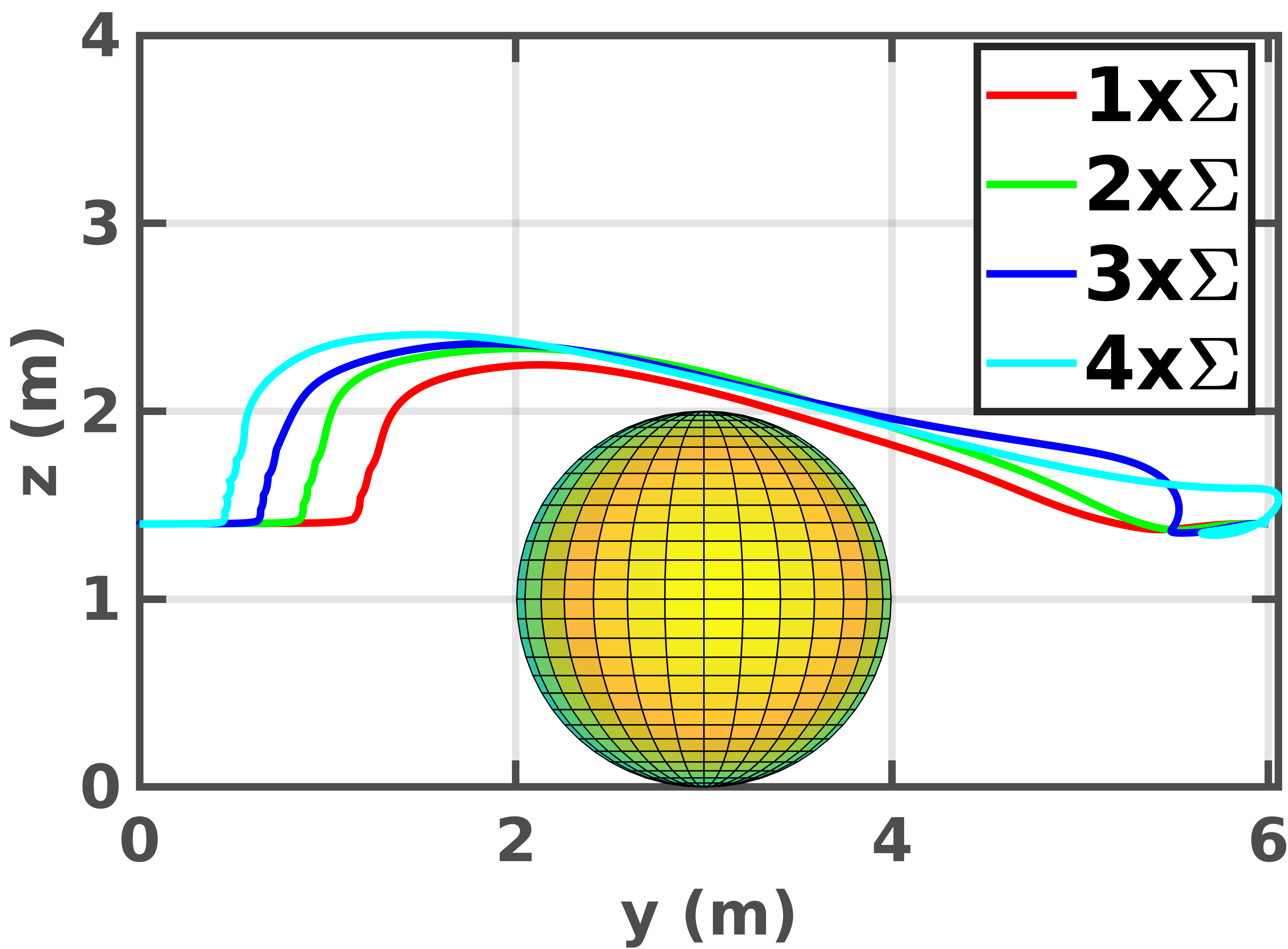}\label{fig:cp4}}\hfill
  \subfloat[Our method]{\includegraphics[scale=0.19]{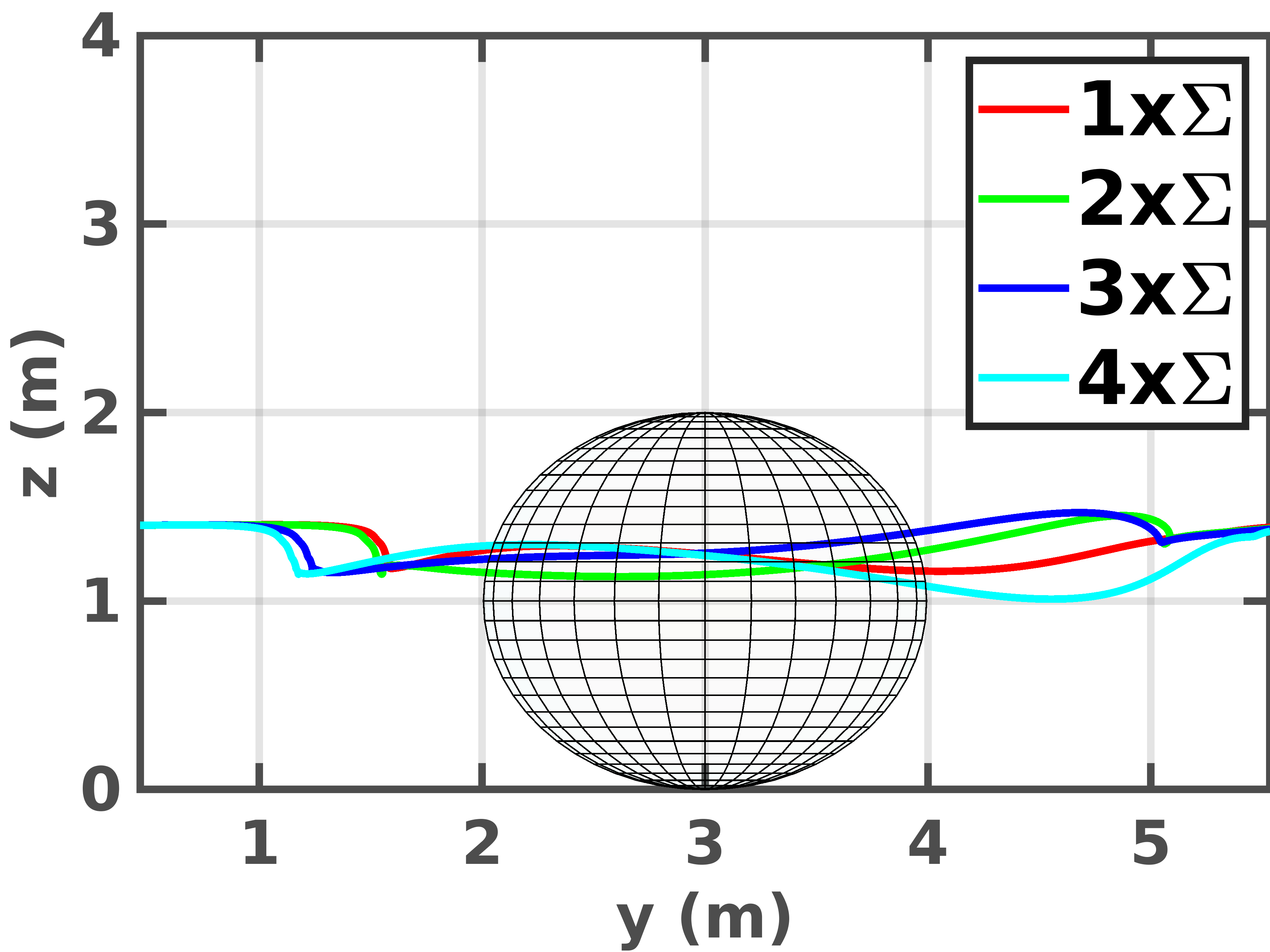}\label{fig:cp5}}\hfill
  \subfloat[Center point approximation]{\includegraphics[scale=0.19]{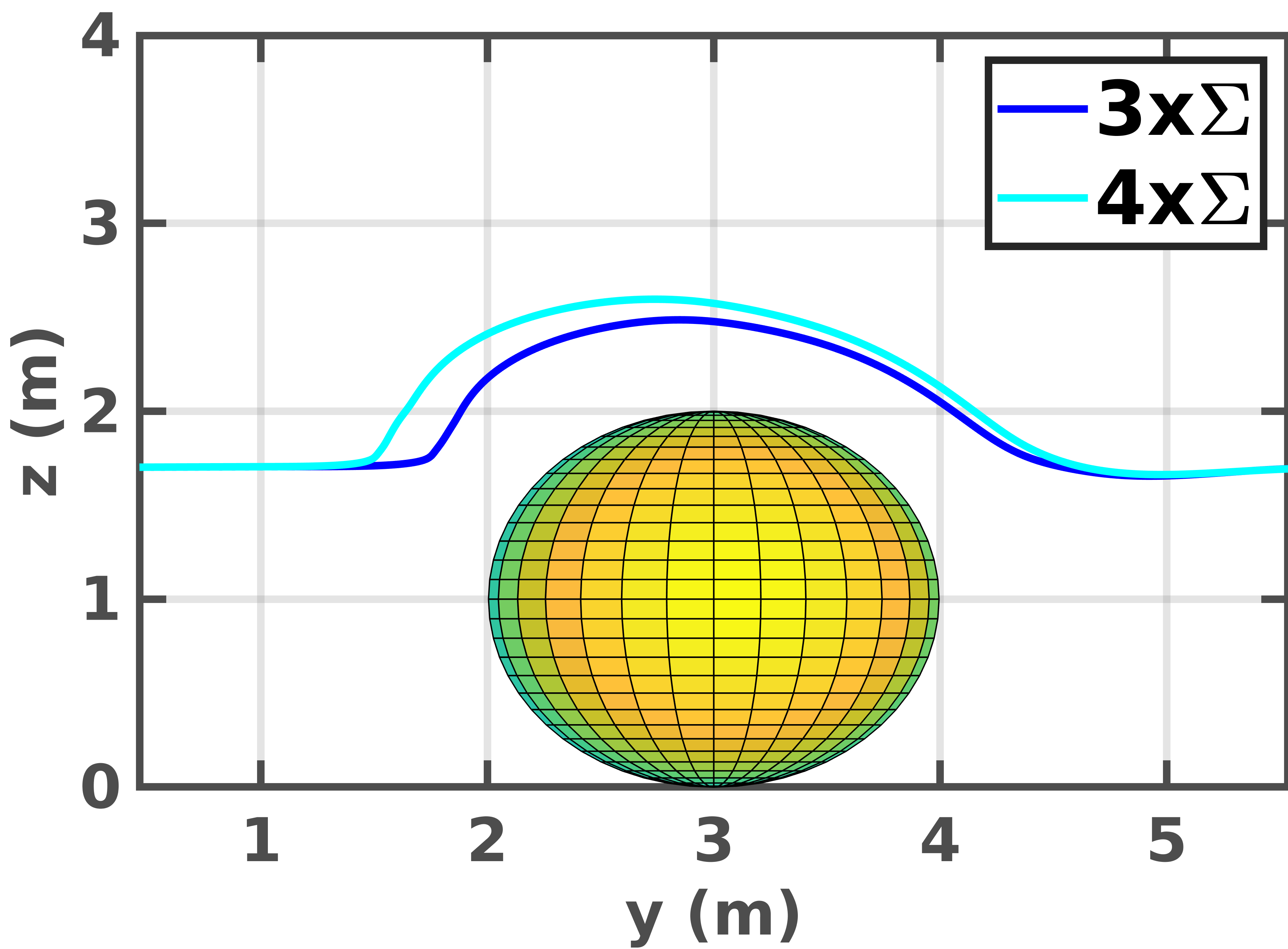}\label{fig:cp6}}
  \vspace{-0.1cm}
  \caption{Simulation with varying measurement noise. The upper plots shows the top view (x-y) and the lower plots show the side view (y-z). The solid lines represent the trajectories executed.}
  \label{fig:cp}
\end{figure}
In this section we describe our implementation and then evaluate the capabilities of our proposed approach. Simulations are performed in Gazebo with a quadrotor of semi-principle axes $(0.18,0.18,0.06)$ m. We refer the readers to~\cite{falanga2018IROS} for the quadrotor dynamics. The ground truth odometry from Gazebo is used to measure the robot pose, mimicking a motion capture system. This measurement is then corrupted with noise which is zero mean with covariance $\Sigma = diag(0.05 \textrm{m}^2, 0.05 \textrm{m}^2, 0.05 \textrm{m}^2, 0.1 \textrm{deg}^2, 0.1 \textrm{deg}^2,$ $0.1 \textrm{deg}^2)$. The optimization of~(\ref{eq:cost_fn}) is performed using the model predictive control approach developed in~\cite{falanga2018IROS} and modified to meet our requirements. In all the experiments we use $\epsilon = 0.05$ and a look-ahead horizon of two seconds with a discretization of 0.1 seconds. The performance is evaluated on an Intel{\small\textregistered} Core i7-6500U CPU$@$2.50GHz$\times$4 with 8GB RAM under Ubuntu 16.04 LTS. The results reported are averaged over 10 different simulation runs. 
\begin{table*}[t]
\small\sf\centering
 \caption{Collision probability efficiency with varying measurement noise. The minimum distance between the quadrotor and the obstacle is denoted by d (m). l (m) is the total trajectory length and T (s) is the total trajectory duration. sp denotes success percentage.}
 \scalebox{0.75}{
\begin{tabular}{|c|C{1.25cm}C{0.55cm}C{0.55cm}C{0.55cm}|C{1.25cm}C{0.55cm}C{0.55cm}C{0.55cm}|C{1.25cm}C{0.55cm}C{0.55cm}C{0.55cm}|C{1.25cm}C{0.55cm}C{0.55cm}C{0.55cm}|  }
\hline
 {} & \multicolumn{16}{|c|}{Measurement noise} \\\cline{2-17}
 {}  & \multicolumn{4}{|c|}{$\Sigma$} & \multicolumn{4}{|c|}{2$\Sigma$} & \multicolumn{4}{|c|}{3$\Sigma$} & \multicolumn{4}{|c|}{4$\Sigma$}\\
 \hline
 Method & d & l & T & sp & d & l & T & sp & d & l & T & sp&d & l &T & sp\\
 \hline
 Bounding volume~\cite{park2012ICAPS,kamel2017IROS} &0.70$\pm$0.07 & 7.72 &6.6 & 100& 0.93$\pm$0.09 & 8.15 & 6.19 &100 & 1.38$\pm$0.14& 8.78 & 6.54&100 &1.51$\pm$0.17 & 9.54 & 7.04&100 \\
 \hline
  Our method & 0.40$\pm$0.05 & 7.48 & 6.01 &100 & 0.45$\pm$0.06 & 8.05 & 6.13&100 & 0.76$\pm$0.07 & 8.21 & 6.19&100 & 0.79$\pm$0.09 & 8.37 & 6.43 &100\\
  \hline
 Center point approximation~\cite{dutoit2011IEEE} & - & - & -& 0& - & - &- &0& 0.10$\pm$0.02 & 7.02 & 6.01& 55  &0.13$\pm$0.02 & 7.10 & 6.08& 35 \\
 \hline
\end{tabular}}
 \label{table:result1}
\end{table*}

\noindent \textit{Comparison to bounding volume approaches}: We compare our approach to bounding volume methods~\cite{park2012ICAPS,kamel2017IROS} which enlarge robots and obstacles with their 3-$\sigma$ confidence ellipsoids. Computation time and complexity are greatly reduced with bounding volumes. However, plans tend to be overly conservative and suboptimal. In the experiment, we consider a quadrotor at $(0,0,1.4)$ m moving to the goal location at $(0,13,1.4)$ m. An obstacle of semi-principle axes $(1,1,1)$ m is located in between at $(0,3,3)$ m. We conduct the experiment with varying measurement noise of $\Sigma, 2\Sigma, 3\Sigma, \textrm{and }  4\Sigma$. To compare the efficiency of our approach we define the compute the following metrics: (a) d-- minimum distance between the quadrotor and the obstacle, (b) l-- total trajectory length, and (c) T-- total trajectory duration. 

The results can be seen in the first two rows of Table~\ref{table:result1}. In all the cases, our approach is more efficient as
can be seen from the shorter average trajectory length and duration, which is more evident as the measurement noise increases. We also provide a comparison to center point approximation~\cite{dutoit2011IEEE}, which is also computationally less intensive. However, as recognized in~\cite{park2018IEEE}, if the covariance is small, the approximated probability can be much smaller than the exact value. Moreover, the approach works well only when the sizes of objects are relatively very small compared with their position uncertainties~\cite{zhu2019RAL}. This is seen in the last row of Table~\ref{table:result1}. For measurement noise $\Sigma$ and $2\Sigma$ the approach resulted in collision for all the runs. For measurement noise $3\Sigma$ and $4\Sigma$, the approach succeeded in 55\% and 35\% of the runs, respectively. This reduced success percentages are due to lower values of the collision probabilities computed. The executed trajectories for all the three approaches in Table~\ref{table:result1} are shown in Fig.~\ref{fig:cp}. For a given $\epsilon$, the metric d allows us to define a measure of risk--- distance to closest obstacle. However, increasing $\epsilon$, increases risk as we solicit controls such that the collision probability is at
most $\epsilon$. For example, in the scenario considered in Fig. 1 (b), an $\epsilon \geq 0.2$ lead to collision in 80$\%$ of the experiments.

\begin{figure*}[]
\vspace{-0.2cm}
  \subfloat[]{\includegraphics[scale=0.23]{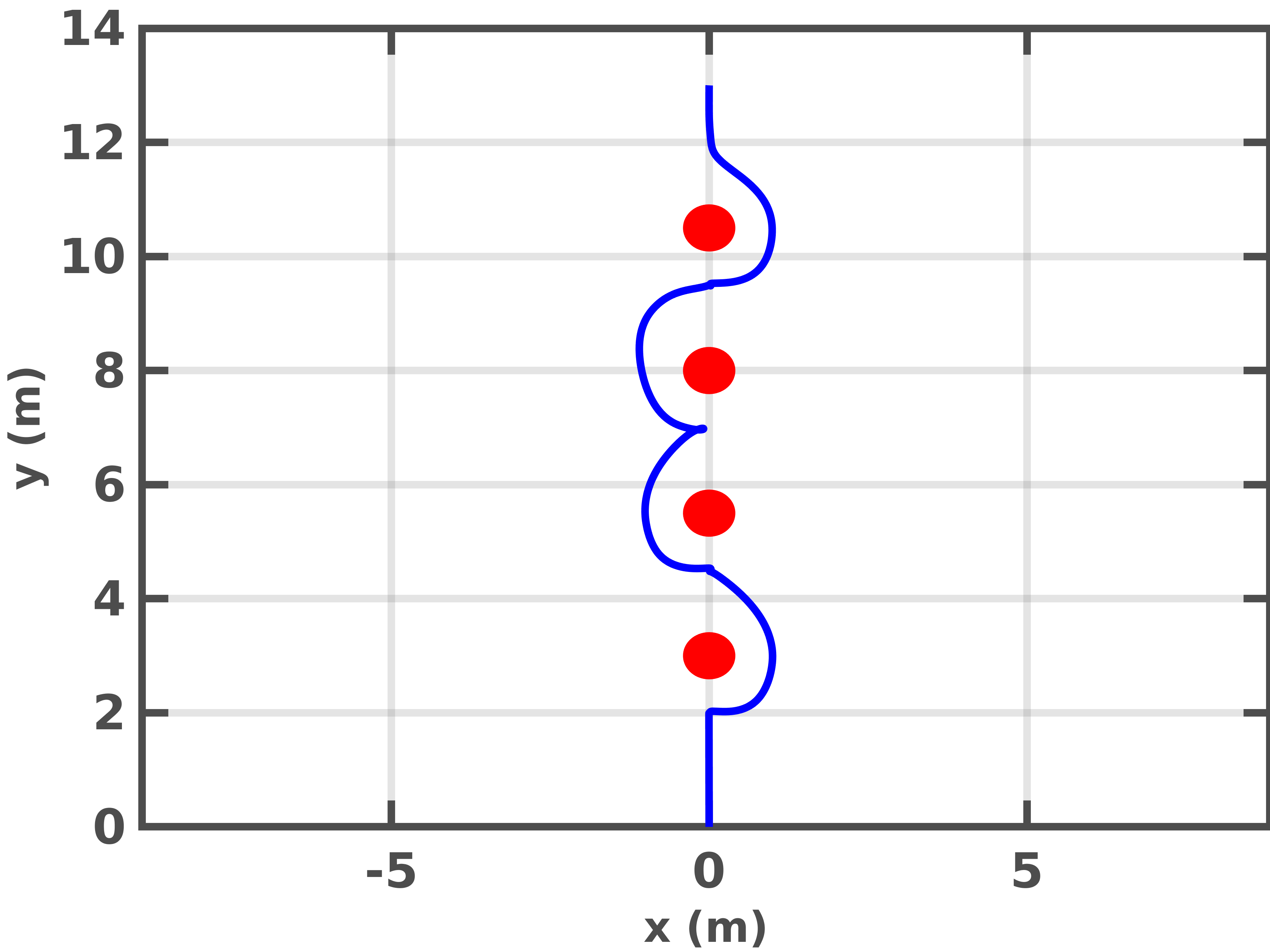}\label{fig:4obs1}}\hfill
  \subfloat[]{\includegraphics[scale=0.23]{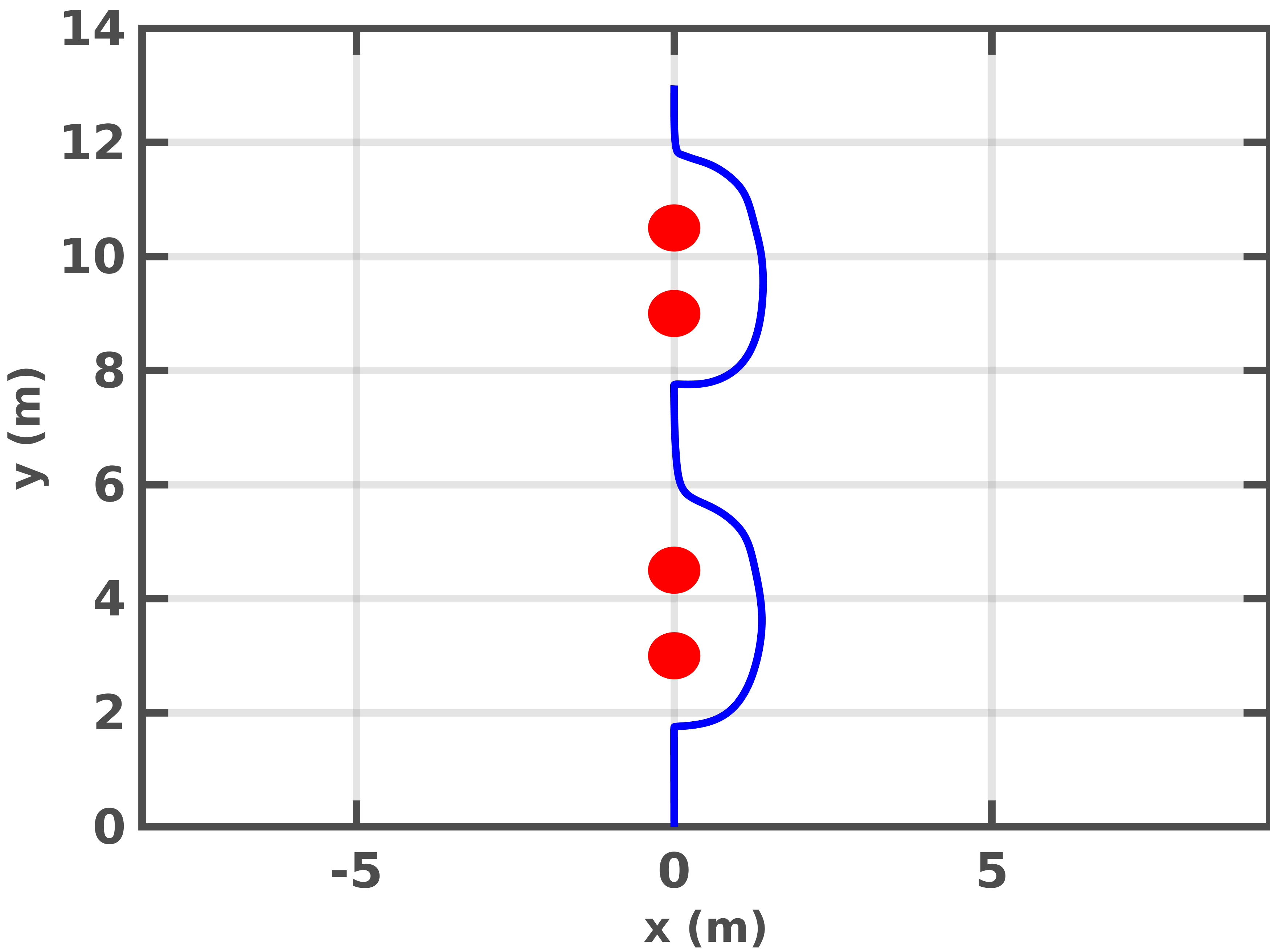}\label{fig:4obs2}}\hfill
  \subfloat[]{\includegraphics[scale=0.23]{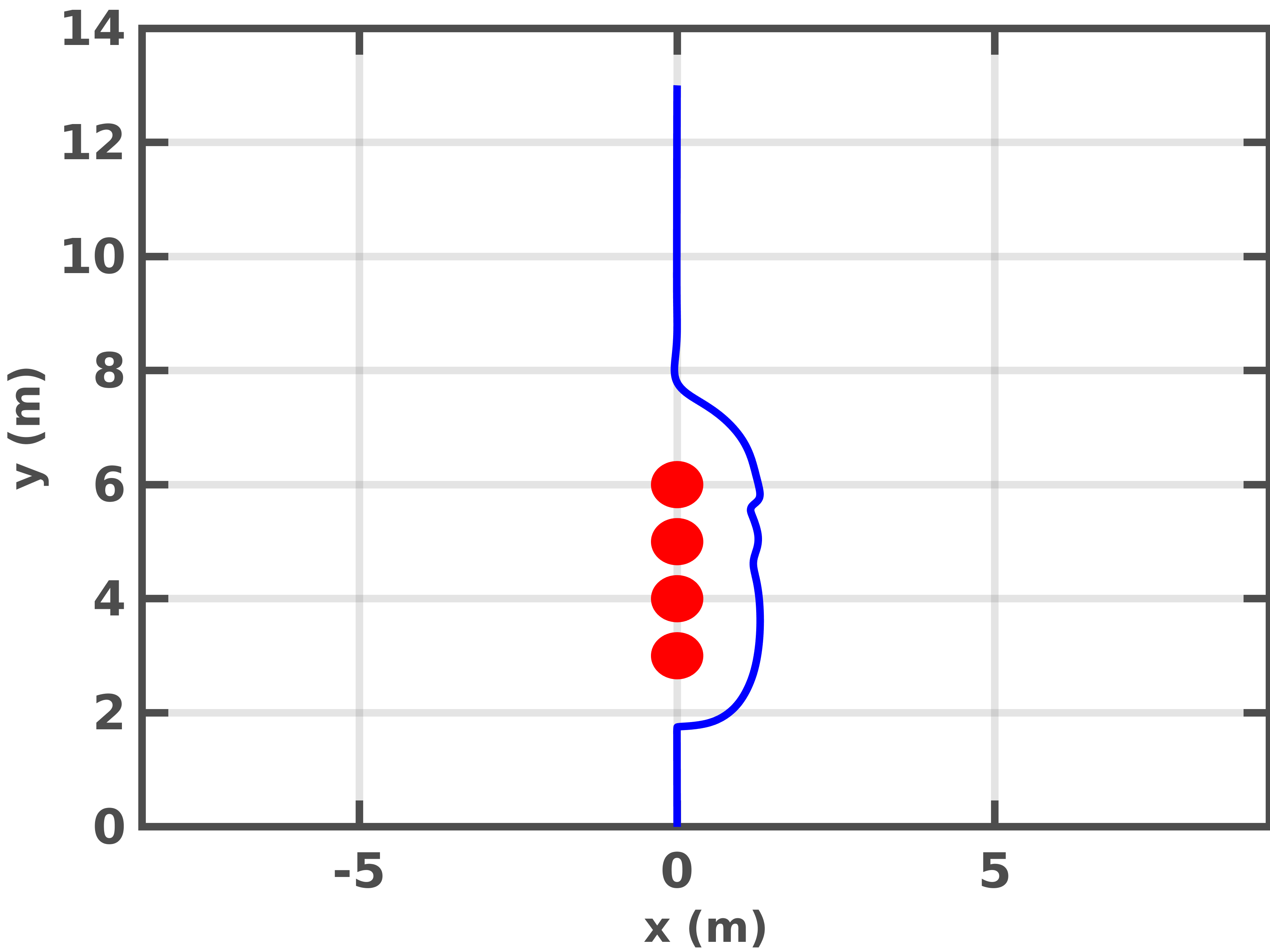}\label{fig:4obs3}}\hfill
  \subfloat[]{\includegraphics[scale=0.23]{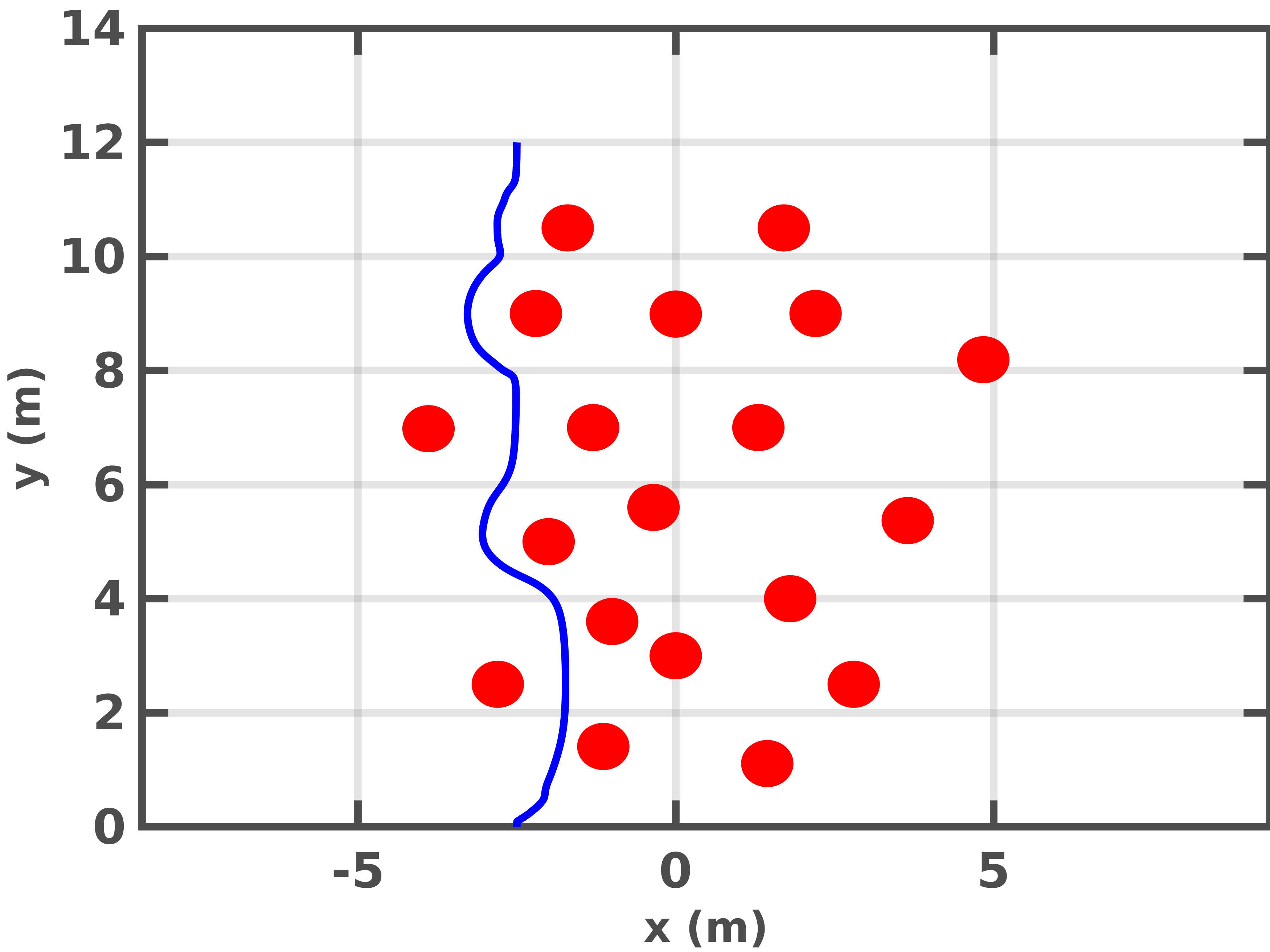}\label{fig:column_traj1}}\hfill
  \subfloat[]{\includegraphics[scale=0.23]{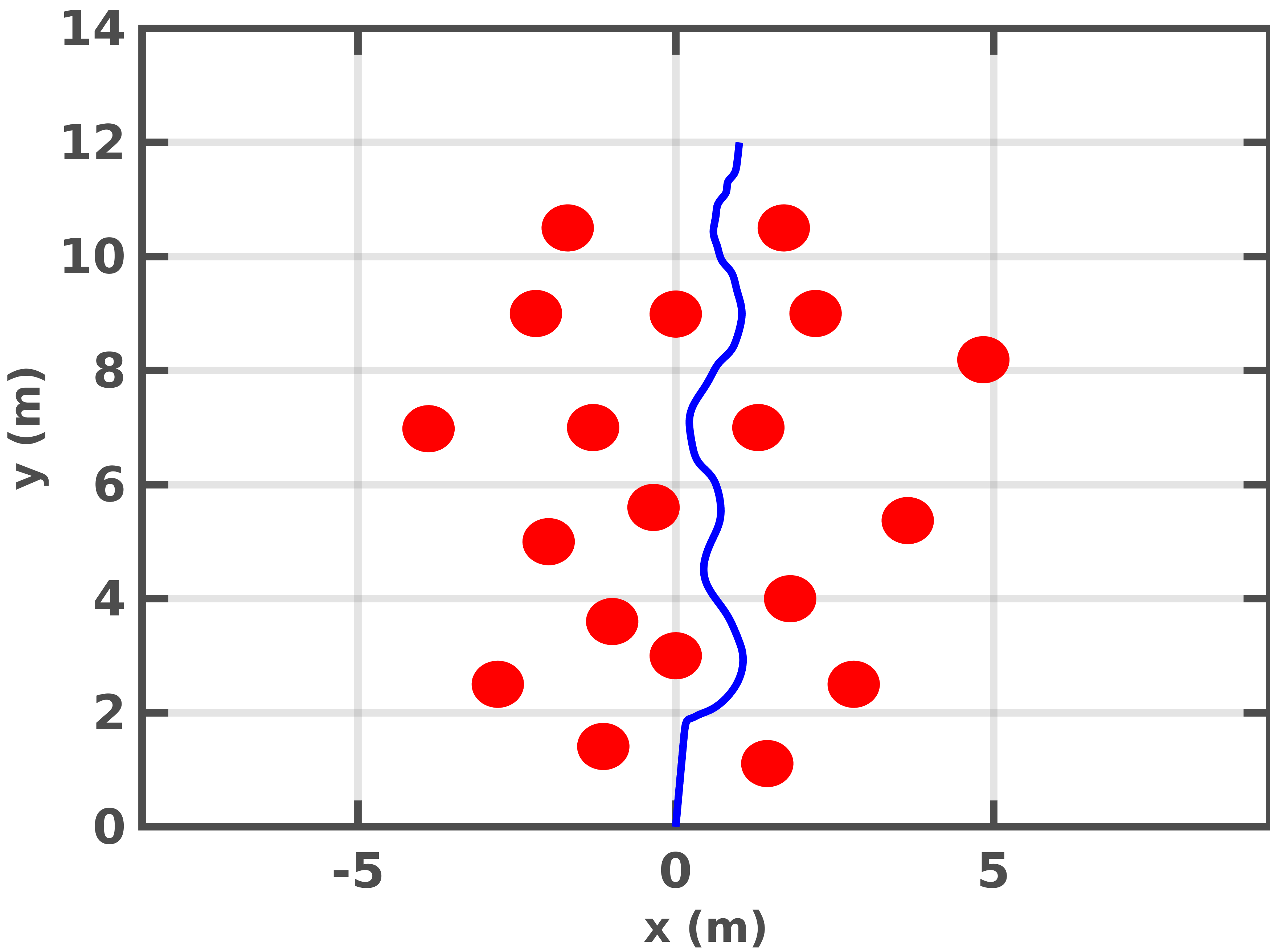}\label{fig:column_traj2}}
  \vspace{-0.1cm}
  \caption{(a)-(c) Top view (x-y) of \textit{four obstacles} in different locations. (d)-(e) Top view of the \textit{column domain}. The solid blue lines represent the trajectories executed by the quadrotor and the red blobs represent the obstacles.}
  \label{fig:4obs}
\end{figure*}
\begin{wraptable}[5]{l}{0.53\linewidth}
\small\sf\centering
\vspace{-0.4cm}
 \scalebox{0.64}{
\begin{tabular}{|c|c|c|c|}
\hline
Obstacle location & l (m) & T (s) & d (m) \\
\hline
(a) &17.78 & 14.94 $\pm$ 0.03 & 0.59\\
\hline
(b) & 16.10 & 14.54 $\pm$ 0.01 & 0.88\\
\hline
(c) & 14.57 & 14.36 $\pm$ 0.02 & 0.82\\
 \hline
\end{tabular}}
  \vspace{-0.2cm}
  \caption{Trajectory results with varying obstacle configurations.}
      \label{table:result2}
\end{wraptable} 
\noindent \textit{Four obstacles}: In this experiment, we consider four obstacles which are placed (a) far apart from each, (b) two obstacles are placed close to each other, and (c) all obstacles are placed close to each other. In each of the cases, the quadrotor starting from $(0,0,1.4)$ m has to reach the goal location at $(0,13,1.4)$ m. The respective trajectories in each of the three cases can be seen in Fig.~\ref{fig:4obs1}-\ref{fig:4obs3}. The change in configuration of the obstacles affects the collision probability computation which is reflected in the respective executed trajectories. The results are shown in Table~\ref{table:result2}.

 \begin{wrapfigure}[9]{l}{0.29\linewidth}
	\centering
	\vspace{-0.4cm}
		\includegraphics[scale=0.075]{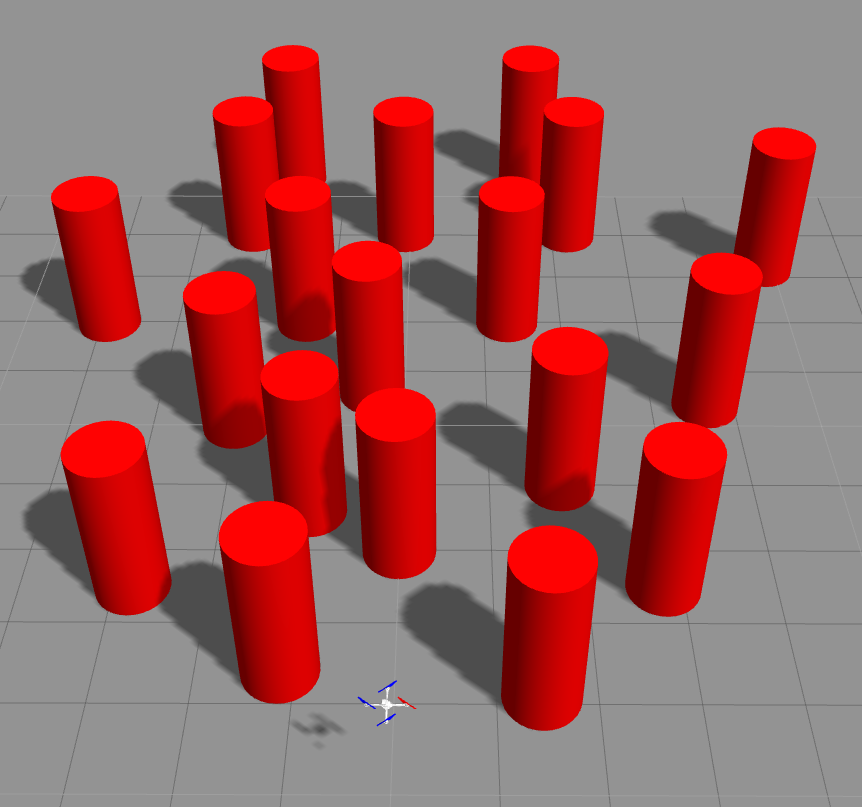}
		\caption{Environment with 19 cylindrical columns.}
	\label{fig:column}
\end{wrapfigure}
\noindent \textit{Column domain}: We test the adaptability of our approach to cluttered and challenging environments. Fig.~\ref{fig:column} shows a snapshot of the column domain which consists of 19 cylindrical columns. We consider two different scenarios where the quadrotor needs to avoid collision with the columns to reach the goal location. In scenario 1, the quadrotor starting from $(-2.5,0,1.4)$ m has to reach the goal location at $(-2.5,$$12,$$1.4)$ m. The trajectory followed is seen in Fig.~\ref{fig:column_traj1}, with an average trajectory length l=13.72 m, and trajectory time T=13.14 ($\pm$0.02) seconds. The average minimum distance between the quadrotor and the obstacles d=0.21 m. In scenario 2, the quadrotor starting from $(0,$$0,$$1.4)$ m has to reach the goal location at $(1,$$12,$$1.4)$ m. The trajectory followed is seen in Fig.~\ref{fig:column_traj2} with l=13.42 m, T=13.37 ($\pm$0.05) seconds and d=0.66 m. Note that since there are 19 columns, there are as many collision avoidance constraints. Our method is computationally less intensive and hence solvable in real time.

%

\section{Conclusion}
In this paper, we present a novel approach to compute an exact as well as a tighter bound for collision probability under Gaussian motion and sensing uncertainties. The collision constraint is formulated as the distance between L\"{o}wner-John ellipsoids of the robot and the obstacles. Efficiency of our approach with respect to trajectory length and duration is tested in simulation by comparing with bounding volume approaches. Trajectory variation due to change in obstacle configurations have also been tested and it is seen that our method readily adapts to the same. Finally, we also test our approach in a heavily cluttered column domain.

Our approach has several limitations. First, we assume a Gaussian parametrization for the belief states. This may not be an admissible approximation in some cases, for example, when obstacles are not known a-priori, one may need to consider multi-modal beliefs.  In such applications one may resort to nonlinear filters (say particle filters) for modeling non-Gaussian beliefs. Yet, in aerial as well as mobile robotics, EKF is extensively used for state estimation since Gaussian distributions are pertinent in a variety of applications. Gaussian distribution is thus a reasonable assumption in any such application. Second, we represent the robot and obstacles using their minimum volume enclosing ellipsoids. For a convex polygon $\mathcal{P} \subseteq \mathbb{R}^n$, $J(\mathcal{P})$ provide an \textit{n-rounding} of $\mathcal{P}$, that is, $n^{-1}J(\mathcal{P}) \subset \mathcal{P} \subset J(\mathcal{P})$; $n^{-\frac{1}{2}}$ for symmetric $\mathcal{P}$. Although the tightness of the bound is still an open problem, for a large class of problems, the approximation of $\mathcal{P}$ using $J(\mathcal{P})$ gives a reasonable first (and sometimes optimal) bound~\cite{giannopoulos2001euclidean}.

\section*{ACKNOWLEDGMENT}
The authors are grateful to the reviewers for their helpful suggestions.

\vspace{-0.1cm}
\bibliographystyle{IEEEtran.bst}
\bibliography{/home/antony/research_genoa/References/References}

\begin{thebibliography}{10}
\providecommand{\url}[1]{#1}
\csname url@rmstyle\endcsname
\providecommand{\newblock}{\relax}
\providecommand{\bibinfo}[2]{#2}
\providecommand\BIBentrySTDinterwordspacing{\spaceskip=0pt\relax}
\providecommand\BIBentryALTinterwordstretchfactor{4}
\providecommand\BIBentryALTinterwordspacing{\spaceskip=\fontdimen2\font plus
\BIBentryALTinterwordstretchfactor\fontdimen3\font minus
  \fontdimen4\font\relax}
\providecommand\BIBforeignlanguage[2]{{%
\expandafter\ifx\csname l@#1\endcsname\relax
\typeout{** WARNING: IEEEtran.bst: No hyphenation pattern has been}%
\typeout{** loaded for the language `#1'. Using the pattern for}%
\typeout{** the default language instead.}%
\else
\language=\csname l@#1\endcsname
\fi
#2}}

\bibitem{dutoit2011IEEE}
N.~E. Du~Toit and J.~W. Burdick, ``Probabilistic collision checking with chance
  constraints,'' \emph{IEEE Transactions on Robotics}, vol.~27, no.~4, pp.
  809--815, 2011.

\bibitem{patil2012ICRA}
S.~Patil, J.~Van Den~Berg, and R.~Alterovitz, ``{Estimating probability of
  collision for safe motion planning under Gaussian motion and sensing
  uncertainty},'' in \emph{IEEE International Conference on Robotics and
  Automation}, 2012, pp. 3238--3244.

\bibitem{park2018IEEE}
C.~Park, J.~S. Park, and D.~Manocha, ``{Fast and bounded probabilistic
  collision detection for high-DOF trajectory planning in dynamic
  environments},'' \emph{IEEE Transactions on Automation Science and
  Engineering}, vol.~15, no.~3, pp. 980--991, 2018.

\bibitem{axelrod2018IJRR}
B.~Axelrod, L.~P. Kaelbling, and T.~Lozano-P{\'e}rez, ``Provably safe robot
  navigation with obstacle uncertainty,'' \emph{The International Journal of
  Robotics Research}, vol.~37, no. 13-14, pp. 1760--1774, 2018.

\bibitem{zhu2019RAL}
H.~Zhu and J.~Alonso-Mora, ``Chance-constrained collision avoidance for mavs in
  dynamic environments,'' \emph{IEEE Robotics and Automation Letters}, vol.~4,
  no.~2, pp. 776--783, 2019.

\bibitem{thomas2021ISR}
A.~Thomas, F.~Mastrogiovanni, and M.~Baglietto, ``{An Integrated Localization,
  Motion Planning and Obstacle Avoidance Algorithm in Belief Space},''
  \emph{Intelligent Service Robotics}, vol.~14, no.~2, pp. 235--250, 2021.

\bibitem{lambert2008ICCARV}
A.~Lambert, D.~Gruyer, and G.~Saint~Pierre, ``{A fast Monte Carlo algorithm for
  collision probability estimation},'' in \emph{10th IEEE International
  Conference on Control, Automation, Robotics and Vision}, 2008, pp. 406--411.

\bibitem{aoude2013AR}
G.~S. Aoude, B.~D. Luders, J.~M. Joseph, N.~Roy, and J.~P. How,
  ``Probabilistically safe motion planning to avoid dynamic obstacles with
  uncertain motion patterns,'' \emph{Autonomous Robots}, vol.~35, no.~1, pp.
  51--76, 2013.

\bibitem{bry2011ICRA}
A.~Bry and N.~Roy, ``Rapidly-exploring random belief trees for motion planning
  under uncertainty,'' in \emph{IEEE International Conference on Robotics and
  Automation}, 2011, pp. 723--730.

\bibitem{kamel2017IROS}
M.~Kamel, J.~Alonso-Mora, R.~Siegwart, and J.~Nieto, ``Robust collision
  avoidance for multiple micro aerial vehicles using nonlinear model predictive
  control,'' in \emph{2017 IEEE/RSJ International Conference on Intelligent
  Robots and Systems (IROS)}.\hskip 1em plus 0.5em minus 0.4em\relax IEEE,
  2017, pp. 236--243.

\bibitem{park2012ICAPS}
C.~Park, J.~Pan, and D.~Manocha, ``{ITOMP: Incremental trajectory optimization
  for real-time replanning in dynamic environments},'' in \emph{Twenty-Second
  International Conference on Automated Planning and Scheduling}, 2012.

\bibitem{lee2013IROS}
A.~Lee, Y.~Duan, S.~Patil, J.~Schulman, Z.~McCarthy, J.~Van Den~Berg,
  K.~Goldberg, and P.~Abbeel, ``Sigma hulls for gaussian belief space planning
  for imprecise articulated robots amid obstacles,'' in \emph{IEEE/RSJ
  International Conference on Intelligent Robots and Systems}, 2013, pp.
  5660--5667.

\bibitem{hardy2013TRO}
J.~Hardy and M.~Campbell, ``Contingency planning over probabilistic obstacle
  predictions for autonomous road vehicles,'' \emph{IEEE Transactions on
  Robotics}, vol.~29, no.~4, pp. 913--929, 2013.

\bibitem{johnson1994truncatedGaussian}
N.~L. Johnson, S.~Kotz, and N.~Balakrishnan, ``Continuous univariate
  distributions. john wiley\& sons,'' \emph{New York, NY}, 1994.

\bibitem{liu2014ICRA}
W.~Liu and M.~H. Ang, ``Incremental sampling-based algorithm for risk-aware
  planning under motion uncertainty,'' in \emph{2014 IEEE International
  Conference on Robotics and Automation (ICRA)}, 2014, pp. 2051--2058.

\bibitem{schmerling2017RSS}
E.~Schmerling and M.~Pavone, ``{Evaluating Trajectory Collision Probability
  through Adaptive Importance Sampling for Safe Motion Planning},'' in
  \emph{Proceedings of Robotics: Science and Systems}, Cambridge,
  Massachusetts, July 2017.

\bibitem{janson2018ISRR}
L.~Janson, E.~Schmerling, and M.~Pavone, ``{Monte Carlo motion planning for
  robot trajectory optimization under uncertainty},'' in \emph{Robotics
  Research}.\hskip 1em plus 0.5em minus 0.4em\relax Springer, 2018, pp.
  343--361.

\bibitem{blackmore2011TRO}
L.~Blackmore, M.~Ono, and B.~C. Williams, ``Chance-constrained optimal path
  planning with obstacles,'' \emph{IEEE Transactions on Robotics}, vol.~27,
  no.~6, pp. 1080--1094, 2011.

\bibitem{sun2016ISRR}
W.~Sun, L.~G. Torres, J.~Van Den~Berg, and R.~Alterovitz, ``Safe motion
  planning for imprecise robotic manipulators by minimizing probability of
  collision,'' in \emph{Robotics Research}.\hskip 1em plus 0.5em minus
  0.4em\relax Springer, 2016, pp. 685--701.

\bibitem{frey2020RSS}
K.~Frey, T.~Steiner, and J.~How, ``{Collision Probabilities for Continuous-Time
  Systems Without Sampling},'' in \emph{Proceedings of Robotics: Science and
  Systems}, Corvalis, Oregon, USA, July 2020.

\bibitem{shimanuki2018WAFR}
L.~Shimanuki and B.~Axelrod, ``{Hardness of 3D Motion Planning Under Obstacle
  Uncertainty},'' \emph{Workshop on Algorithmic Foundations of Robotics}, 2018.

\bibitem{salzman2017ICAPS}
O.~Salzman, B.~Hou, and S.~Srinivasa, ``Efficient motion planning for problems
  lacking optimal substructure,'' in \emph{Twenty-Seventh International
  Conference on Automated Planning and Scheduling}, 2017.

\bibitem{jasour2019RSS}
A.~M. Jasour and B.~C. Williams, ``Risk contours map for risk bounded motion
  planning under perception uncertainties,'' \emph{Robotics: Science and
  Systems}, 2019.

\bibitem{hakobyan2019RAL}
A.~Hakobyan, G.~C. Kim, and I.~Yang, ``{Risk-Aware Motion Planning and Control
  Using CVaR-Constrained Optimization},'' \emph{IEEE Robotics and Automation
  Letters}, vol.~4, no.~4, pp. 3924--3931, 2019.

\bibitem{bajcsy2019ICRA}
A.~Bajcsy, S.~L. Herbert, D.~Fridovich-Keil, J.~F. Fisac, S.~Deglurkar, A.~D.
  Dragan, and C.~J. Tomlin, ``A scalable framework for real-time multi-robot,
  multi-human collision avoidance,'' in \emph{2019 international conference on
  robotics and automation (ICRA)}.\hskip 1em plus 0.5em minus 0.4em\relax IEEE,
  2019, pp. 936--943.

\bibitem{fridovich2020IJRR}
D.~Fridovich-Keil, A.~Bajcsy, J.~F. Fisac, S.~L. Herbert, S.~Wang, A.~D.
  Dragan, and C.~J. Tomlin, ``Confidence-aware motion prediction for real-time
  collision avoidance1,'' \emph{The International Journal of Robotics
  Research}, vol.~39, no. 2-3, pp. 250--265, 2020.

\bibitem{ding2013ICRA}
X.~C. Ding, A.~Pinto, and A.~Surana, ``Strategic planning under uncertainties
  via constrained markov decision processes,'' in \emph{IEEE International
  Conference on Robotics and Automation}, 2013, pp. 4568--4575.

\bibitem{sadigh2016RSS}
D.~Sadigh and A.~Kapoor, ``Safe control under uncertainty with probabilistic
  signal temporal logic,'' \emph{Robotics: Science and Systems}, 2016.

\bibitem{john1948extremum}
F.~John, \emph{{Extremum problems with inequalities as subsidiary conditions,
  Studies and Essays Presented to R. Courant on his 60th Birthday, January 8,
  1948}}.\hskip 1em plus 0.5em minus 0.4em\relax Interscience Publishers, Inc.,
  New York, reprinted in: J. Moser (ed.), Fritz John Collected Papers, Ch. 2,
  Birkhauser, Boston, pages 543-560, 1985.

\bibitem{rimon1997JINT}
E.~Rimon and S.~P. Boyd, ``Obstacle collision detection using best ellipsoid
  fit,'' \emph{Journal of Intelligent and Robotic Systems}, vol.~18, no.~2, pp.
  105--126, 1997.

\bibitem{provost1992book}
S.~B. Provost and A.~Mathai, \emph{Quadratic forms in random variables: theory
  and applications}.\hskip 1em plus 0.5em minus 0.4em\relax M. Dekker, 1992.

\bibitem{thomas2020IRIM}
A.~Thomas, F.~Mastrogiovanni, and M.~Baglietto, ``{Motion Planning with
  Environment Uncertainty},'' in \emph{Italian Conference on Robotics and
  Intelligent Machines (I-RIM)}, 2020.

\bibitem{prentice2009IJRR}
S.~Prentice and N.~Roy, ``{The belief roadmap: Efficient planning in belief
  space by factoring the covariance},'' \emph{The International Journal of
  Robotics Research}, vol.~28, no. 11-12, pp. 1448--1465, 2009.

\bibitem{van_den_berg2012IJRR}
J.~Van Den~Berg, S.~Patil, and R.~Alterovitz, ``Motion planning under
  uncertainty using iterative local optimization in belief space,'' \emph{The
  International Journal of Robotics Research}, vol.~31, no.~11, pp. 1263--1278,
  2012.

\bibitem{agha_mohammadi2014IJRR}
A.-A. Agha-Mohammadi, S.~Chakravorty, and N.~M. Amato, ``{FIRM: Sampling-based
  feedback motion-planning under motion uncertainty and imperfect
  measurements},'' \emph{The International Journal of Robotics Research},
  vol.~33, no.~2, pp. 268--304, 2014.

\bibitem{pathak2017ICRA}
S.~Pathak, A.~Thomas, and V.~Indelman, ``Nonmyopic data association aware
  belief space planning for robust active perception,'' in \emph{2017 IEEE
  International Conference on Robotics and Automation (ICRA)}.\hskip 1em plus
  0.5em minus 0.4em\relax IEEE, 2017, pp. 4487--4494.

\bibitem{mendel1995estimation}
J.~M. Mendel, \emph{Lessons in estimation theory for signal processing,
  communications, and control}.\hskip 1em plus 0.5em minus 0.4em\relax Pearson
  Education, 1995.

\bibitem{indelman2015IJRR}
V.~Indelman, L.~Carlone, and F.~Dellaert, ``{Planning in the Continuous Domain:
  a Generalized Belief Space Approach for Autonomous Navigation in Unknown
  Environments},'' \emph{International Journal of Robotics Research}, vol.~34,
  no.~7, pp. 849--882, 2015.

\bibitem{pathak2018IJRR}
S.~Pathak, A.~Thomas, and V.~Indelman, ``{A unified framework for data
  association aware robust belief space planning and perception},'' \emph{The
  International Journal of Robotics Research}, vol.~37, no. 2-3, pp. 287--315,
  2018.

\bibitem{thomas2019ISRR}
A.~Thomas, F.~Mastrogiovanni, and M.~Baglietto, ``{Task-Motion Planning for
  Navigation in Belief Space},'' in \emph{The International Symposium on
  Robotics Research}, 2019.

\bibitem{falanga2018IROS}
D.~Falanga, P.~Foehn, P.~Lu, and D.~Scaramuzza, ``{PAMPC: Perception-aware
  model predictive control for quadrotors},'' in \emph{2018 IEEE/RSJ
  International Conference on Intelligent Robots and Systems (IROS)}.\hskip 1em
  plus 0.5em minus 0.4em\relax IEEE, 2018, pp. 1--8.

\bibitem{giannopoulos2001euclidean}
A.~A. Giannopoulos and V.~D. Milman, ``Euclidean structure in finite
  dimensional normed spaces,'' \emph{Handbook of the geometry of Banach
  spaces}, vol.~1, pp. 707--779, 2001.

\end{thebibliography}

\end{document}